%% file: main.tex
\documentclass{article}

\usepackage[final]{neurips_2023}

\usepackage[utf8]{inputenc} 
\usepackage[T1]{fontenc}    
\usepackage{hyperref}       
\usepackage{url}            
\usepackage{booktabs}       
\usepackage{amsfonts}       
\usepackage{nicefrac}       
\usepackage{microtype}      
\usepackage{xcolor}         
\usepackage{multirow}
\usepackage{adjustbox}
\usepackage{graphicx}

\usepackage{caption}
\usepackage{subcaption}
\usepackage{amsmath}
\usepackage{amsthm}
\usepackage{xspace}
\usepackage{pifont}
\usepackage{wrapfig}
\usepackage{natbib}

\usepackage{colortbl} 
\definecolor{greenHighlight}{rgb}{0.7,1,0.7} 

\setcitestyle{round}

\newcommand{\model}{\textsc{Llama 2 Long}\xspace}
\newcommand{\modelchat}{\textsc{Llama 2 Long Chat}\xspace}
\newcommand{\llamavtwo}{\textsc{Llama~2}\xspace}
\newcommand{\llama}{\textsc{Llama}\xspace}
\newcommand{\chatllama}{\textsc{Llama~2~Chat}\xspace}

\newcommand\extrafootertext[1]{%
    \bgroup
    \renewcommand\thefootnote{\fnsymbol{footnote}}%
    \renewcommand\thempfootnote{\fnsymbol{mpfootnote}}%
    \footnotetext[0]{#1}%
    \egroup
}

\newtheorem{theorem}{Theorem}

\title{Effective Long-Context Scaling of Foundation Models}

\author{Wenhan Xiong$^{\dagger\ast}$, 
 Jingyu Liu$^\dagger$, Igor Molybog,\\[0.2cm]
\textbf{Hejia Zhang, Prajjwal Bhargava, Rui Hou, Louis Martin, Rashi Rungta,} \\ \textbf{Karthik Abinav Sankararaman, Barlas Oguz, Madian Khabsa, Han Fang,} \\
\textbf{Yashar Mehdad, Sharan Narang, Kshitiz Malik, Angela Fan,} \\[0.2cm]
 \textbf{Shruti Bhosale, Sergey Edunov, Mike Lewis, Sinong Wang$^\ast$, Hao Ma$^\ast$}\\ \\
 \hspace*{0pt} GenAI, Meta
}

\begin{document}

\maketitle

\vspace{-2em}

\begin{abstract}


We present a series of long-context LLMs that support effective context windows of up to 32,768 tokens. 
Our model series
are built through continual pretraining from \llamavtwo with longer training sequences and on a dataset where long texts are upsampled. We perform extensive evaluation on 
language modeling, synthetic context probing tasks, and a wide range of research benchmarks.
On research benchmarks, 
our models achieve consistent improvements on most regular tasks and significant improvements on long-context tasks over \llamavtwo. Notably, with a cost-effective instruction tuning procedure that does not require human-annotated long instruction data,
the 70B variant
can already surpass \texttt{gpt-3.5-turbo-16k}'s overall performance on a suite of long-context tasks. Alongside these results, we provide an in-depth analysis on the individual components of our method. We delve into \llama's position encodings and discuss its limitation in modeling long dependencies. We also examine the impact of various design choices in the pretraining process, including the data mix and the training curriculum of sequence lengths -- our ablation experiments suggest that having abundant long texts in the pretrain dataset is \textit{not} the key to achieving strong performance, and we empirically verify that long context continual pretraining is more efficient and similarly effective compared to pretraining from scratch with long sequences. 






\end{abstract}

\begin{figure}[h!]
    \centering
    \includegraphics[width=0.75\linewidth]{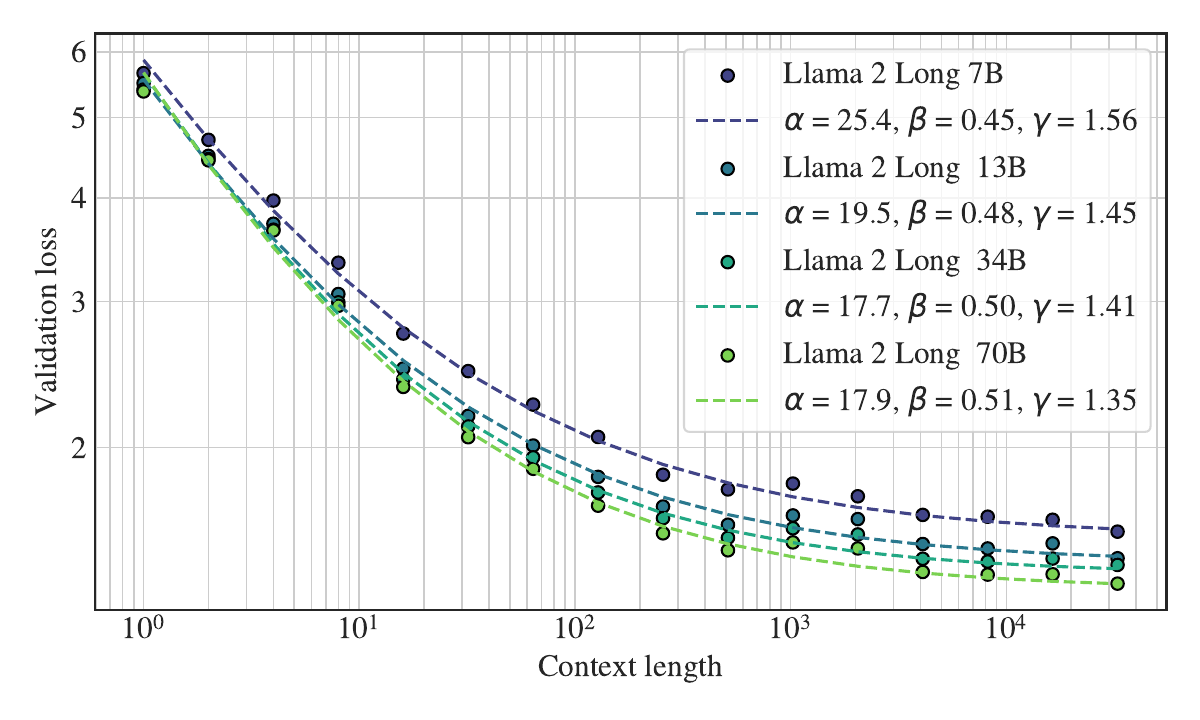}
    \caption{We show that our model's validation loss can be fit as a function of the context length: $L(c) = (\frac{\alpha}{c})^\beta + \gamma$ with a different set of $\alpha, \beta, \gamma$ for each model size. This power-law relationship also suggests that context length is another important axis of scaling LLMs and our model can continually improve its performance as we increase the context length up to 32,768 tokens.} 
    \vspace{-10pt}
    
    \label{fig:scaling_law}
\end{figure}


\input{sections/intro}

\input{sections/longllama}


\section{Main Results}

\input{sections/evaluation}

\section{Analysis}

In this section. We perform ablation experiments to justify our design choices (i.e. architecture modification, data mixes, and training curriculum) and quantify their contributions to the final performance. 

\input{sections/pe}

\input{sections/pretrain_data}

\input{sections/sft}

\input{sections/curriculum}

\section{AI Safety}
\input{sections/safety}


\section{Limitations}

\paragraph{Limited Functionality.} The our model proposed in this paper has not yet been finetuned for a wide range of long-context applications, such as creative writing that require long-form outputs. Applying existing alignment recipes, e.g., RLHF, for various scenarios is expensive and nontrivial. Even skilled annotators may struggle to the intricate details in dense texts. In this regard, we consider developing efficient alignment methods for long LLMs to be a very valuable direction for future research. 

    
\paragraph{Tokenizer Efficiency.} While the proposed our model series can consume contexts up to 32,768 tokens, the actually number of words our model can take is largely affected by the tokenizer behaviour. The tokenizer used by the Llama series has a relatively small vocabulary (32k symbols) and often produces longer sequences compare to the sequences given by GPT-3.5's tokenizer -- we observe our tokenizer often produce $10\%$ more tokens on average. Additionally, the tokenizer we use also cannot efficiently handle whitespace, making it inefficient to process long code data.


\paragraph{Hallucination.} Like other LLMs, we have observed hallucination issue when testing the proposed our model. While this issue is common for short-context models, tackling with this problem for long-context models can be more pronounced because of the dense information they consume and the insufficient alignment process. 

\section{Conclusion}
\input{sections/conclusion}

\section{Acknowledgement}

We would like to thank Nikolay Bashlykov, Matt Wilde, Wenyin Fu, Jianyu Huang, Jenya Lee, Mathew Oldham, and Shawn Xu for their invaluable support on the data, infrastructure, and various other aspects of this project.

\medskip

{
    \small
    \bibliographystyle{plainnat}
    \bibliography{main}
}

\input{sections/appendix}


\end{document}

%% file: sections/intro.tex
\section{Introduction}

Large language models (LLMs), trained with an unprecedented magnitude of data and compute, 
hold the promise of fundamentally improving the way we interact with the digital world.
As LLMs get rapidly deployed and continue to evolve through scaling, we envision these models to serve more intricate and complex use cases, such as analyzing dense knowledge-rich documents, powering more genuine and engaging chatbot experiences, and aiding human users in iterative creation processes such as coding and design, etc. A crucial feature supporting this evolution is the ability to effectively process long-context inputs. 

Until now, LLMs with robust long-context capabilities are primarily provided through proprietary LLM APIs~\citep{claude,gpt-4} and there is no open recipe for building long-context model that can
demonstrate on-par downstream performance as these proprietary models. Moreover, existing open-sourced long-context models~\citep{long_llama,pi,passkey,MPT7b} often fall short on evaluations and primarily measure long-context capabilities with the language modeling loss and synthetic tasks, which do not comprehensively demonstrate their effectiveness in diverse, real-world scenarios. Additionally, these models often overlook the necessity of maintaining strong performance on standard short-context tasks, either bypassing the evaluations or reporting degenerated performance~\citep{yarn,pi}. 
\extrafootertext{$^\dagger$ Equal contribution} 
\extrafootertext{$^\ast$ Corresponding authors:\{xwhan, sinongwang, haom\}@meta.com} 




In this work, we describe our approach to build long-context LLMs with superior performance over all existing open-sourced models. We build our models by continually pretraining from \llamavtwo checkpoints with additional 400 billion tokens formed as long training sequences. Among the model series, the smaller 7B/13B variants are trained with 32,768-token sequences while the 34B/70B variants with 16,384-token sequences. In contrast to the limited evaluation performed by existing studies, we extensively evaluate our models using language modeling, synthetic tasks, and also a wide range of real-world benchmarks covering both long and short context tasks. On language modeling, our model demonstrates a clear power-law scaling behavior with respect to context lengths. This scaling behavior, as shown in Figure~\ref{fig:scaling_law}, not only shows our models' ability to consistently benefit from more contexts but also suggest that context length is another importance axis of scaling LLMs. When comparing our models to \llamavtwo on research benchmarks, we not only observe significant improvements on long-context tasks but also modest improvements on standard short-context tasks, especially on coding, math, and knowledge benchmarks. We explored using a simple and cost-effective procedure to instruction finetune our continually pretrained long models without any human-annotated data. The end result is a chat model that can achieve stronger overall performance than \texttt{gpt-3.5-turbo-16k} on a series of long-context benchmarks covering question answering, summarization, and multi-document aggregation tasks.


In the remaining part of this paper, we begin by presenting the continual long-context pretraining approach and a lightweight instruction tuning procedure, followed by detailed results on a range of short and long context tasks. To facilitate future research, we complement our results with an analysis section discussing how the design of positional encodings, the length distribution of the dataset and the training curriculum contributes to the final performance. Finally, we report responsible safety evaluations, which validates that our models can largely maintain the safety performance of the original \llamavtwo series. 

%% file: sections/longllama.tex
\section{Method}
\label{sec:method}

\subsection{Continual Pretraining}

Training with longer sequence lengths can introduce significant computational overhead due to the quadratic attention calculations. This is the main motivation of our continual pretraining approach. 
The underlying hypothesis that similar long-context capabilities can be learned by continually pretraining from a short-context model is later validated in Section~\ref{sec:curriculum_analysis} 
through comparing different training curricula. 
We keep the original \llamavtwo architecture nearly intact for continual pretraining and only make a necessary modification to the positional encoding that is crucial for the model to attend longer. We also choose not to apply sparse attention~\citep{blocksparse} in this work, since given \llamavtwo 70B's model dimension ($h$ = 8192), the cost of attention matrix calculation and value aggregation only becomes a computation bottleneck when the sequence length exceeds 49,152 ($6h$) tokens~\citep{megatron}.\footnote{While sparse attention might be useful for reducing the key/value cache size at inference time when trading off performance, it can complicate the inference pipeline and the improvements can also be offset by quantization methods.}

\paragraph{Positional Encoding} Through early experiments at the 7B scale, we identified a key limitation of \llamavtwo's positional encoding (PE) that prevents the attention module from aggregating information of distant tokens. We adopt a minimal yet necessary modification on the RoPE positional encoding~\citep{rope} for long-context modeling -- decreasing the rotation angle (controlled by the hyperparameter ``base frequency $b$''), which reduces the decaying effect of RoPE for distant tokens. 
In Section~\ref{pe_analysis}, we show this simple method outperforms a concurrent approach~\citep{pi} for extending \llama's context length and provide a theoretic explanation of its superiority.  

\paragraph{Data Mix} On top of the working model with the modified PE, we further explored different pretrain data mixes in Section~\ref{datamix_analysis} for improving long-context abilities, either by adjusting the ratio of \llamavtwo's pretraining data or adding new long text data. We found that often the quality of the data plays a more critical role than the length of texts for long-context continual pretraining. 

\paragraph{Optimization Details} We continually pretrain \llamavtwo checkpoints with increased sequence length while keeping the same number of tokens per batch as in \llamavtwo. We train all models for a total of 400B tokens over 100,000 steps. 
With \textsc{FlashAttention}~\citep{flashattention}, there is negligible GPU memory overhead as we increase the sequence length and we observe around $17\%$ speed loss when increasing the sequence length from 4,096 to 16,384 for the 70B model. For the 7B/13B models, we use learning rate $2e^{-5}$ and a cosine learning rate schedule with 2000 warm-up steps. For the larger 34B/70B models, we find it important to set a smaller learning rate ($1e^{-5}$) to get monotonically decreasing validation losses. 

\subsection{Instruction Tuning}

Collecting human demonstration and preference labels for LLM alignment is a cumbersome and expensive process~\citep{rlhf,llama2}. The challenge and cost are more pronounced under long-context scenarios, which often involve complex information flow and specialized knowledge, e.g., processing dense legal/scientific documents, making the annotation task nontrivial even for skilled annotators. In fact, most existing open-source instruction datasets~\citep{dolly_v2,oasis} predominantly consist of short samples.

In this work, we found that a simple and cheap approach which leverages a pre-built large and diverse short-prompt dataset works surprisingly well on long-context benchmarks. Specifically, we take the RLHF dataset used in \chatllama and augment it with synthetic self-instruct~\citep{selfinst} long data generated by \chatllama itself, in the hope that the model can learn a diverse set of skills through the large amount of RLHF data and transfer that knowledge to long-context scenarios via self-instruct data. The data generation process focuses on QA-format tasks: starting from a long document in our pretraining corpus, we select a random chunk and prompt \chatllama to write question-answer pairs based on information in the text chunk. We collect both long and short form answers with different prompts. After that, we also adopt a self-critique step where we prompt \chatllama to verify the model-generated answers. Given a generated QA pair, we use the original long document (truncated to fit the model's maximum context length) as the context to construct a training instance. 


For short instruction data, we concatenate them as 16,384-token sequences. For long instruction data, we add padding tokens on the right so that models can process each long instance individually without truncation. While standard instruction tuning only calculates loss on the output tokens, we find it particularly beneficial to also calculate the language modeling loss on the long input prompts, which gives consistent improvements on downstream tasks (Section~\ref{sec:sft}). 


%% file: sections/evaluation.tex
\subsection{Pretrained Model Evaluation}

\input{tables/regular_eval}

\paragraph{Short Tasks} To make long-context LLMs universally useful, an important desiderata is to ensure robust performance on standard short-context tasks. We verify our models' performance on a series of common benchmarks following the previous work~\citep{llama2}. The aggregated results are shown in Table~\ref{table:pretrain_regular_eval}. Overall, we observe \textit{on-par and, in most cases, stronger results} than \llamavtwo. Notably, we observe significantly improved results on coding, math, and knowledge intensive tasks such as MMLU. As shown in Table~\ref{table:compare_openai}, our model outperforms GPT-3.5 on MMLU and GSM8k. This is in contrast to a previous work~\citep{pi} which observes degradation on short tasks. We attribute the improvements to additional computation FLOPs and the knowledge learned from newly introduced long data. 

\input{tables/compare_gpt_3.5}

\input{tables/compare_with_oss}

\begin{figure}[ht]
    \centering
    \includegraphics[width=0.55\linewidth]{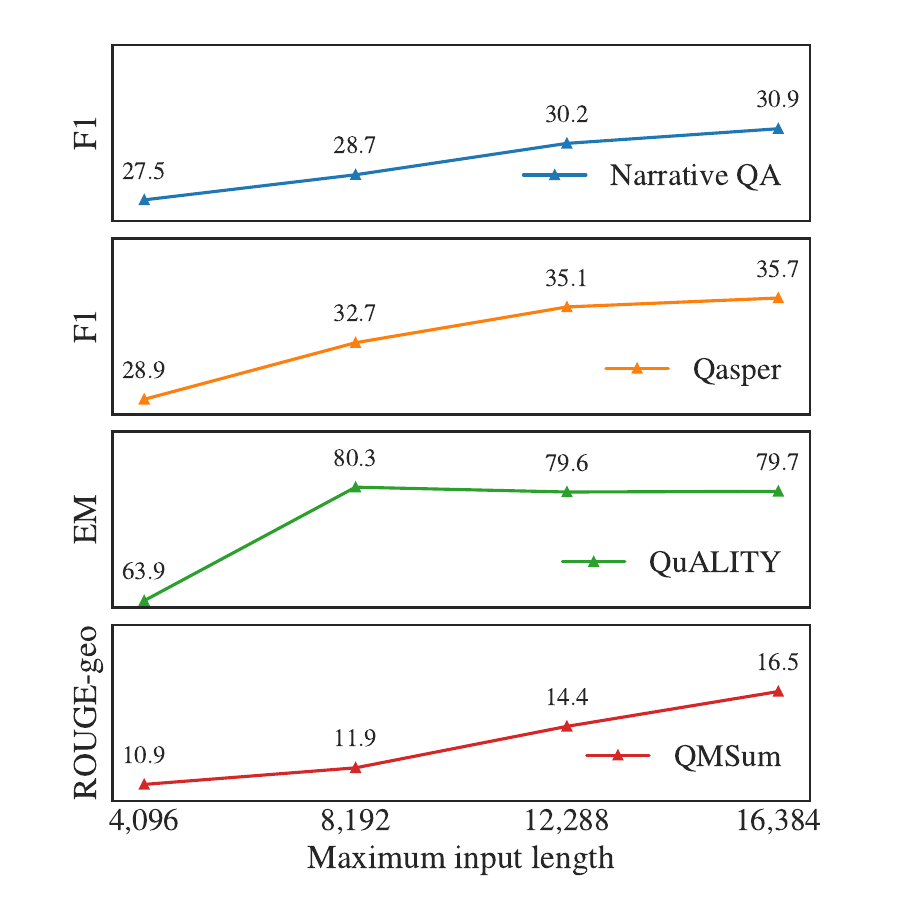}
    \caption{Performance on long-context tasks as the maximum context lengths of prompts increase. 
    }
    \label{fig:scrolls_70B_context_evo}
\end{figure}

\paragraph{Long Tasks} Different from previous works~\citep{pi,passkey} that mostly rely on perplexity and synthetic tasks to gauge long-context performance, we perform long-context evaluation using real-world language tasks. We evaluate 0-shot performance on NarrativeQA~\citep{narrativeqa}, 2-shot on QuALITY~\citep{quality} and Qasper~\citep{qasper}, and 1-shot on QMSum~\citep{qmsum}. The number of shots are decided based on the average sample length of each dataset (i.e., samples in Qasper and QuALITY are often much shorter than those of NarrativeQA). We focus these QA-style tasks because of the ease of prompt engineering\footnote{We use simple prompt ``\{\textsc{context}\} Q: \{\textsc{question}\}, A:'' to evaluate all pretrained models.} and less biased automatic evaluations.
The input prompts are truncated from the left side if the prompts exceed the maximum input length of the model or 16,384 tokens. We compare with open-source long-context models available in Huggingface Transformers, namely Focused Transformer~\citep{focused}, YaRN~\citep{yarn}, Xgen~\citep{xgen}, MPT~\citep{MPT7b,MPT30b} and Together's \llamavtwo fork~\citep{together}. As shown in Table~\ref{table:compare_with_oss}, 
our models achieve superior performance compared to these models. At the 7B scale, only ``Together-7B-32k'' can match our model's performance. Note that this model is not a purely self-supervised model and has been finetuned using a large supervised dataset to improve its few-shot results. 
As the 7/13B variants of our models have been trained with 32k-token sequences, we also perform comparisons using 32,768 maximum prompts lengths and the results are consistent, as shown in Table~\ref{table:compare_with_oss_32k}.



\paragraph{Effective Context Utilization} To validate that our models can effectively use increased context window, we first show in Figure~\ref{fig:scrolls_70B_context_evo} that the results on each long task improve monotonically as we increase the context lengths. Inspired by ~\citep{kaplan2020scaling, hoffmann2022training}, we also found that the language modeling loss of 
our model follows a power-law plus constant scaling relationship with the context length (Figure~\ref{fig:scaling_law}), suggesting: 

\begin{itemize}
    \item Our model continues to show gain in performance (on the language modeling loss) up to 32,768 tokens of text, despite having diminishing returns.
    Taking 
    our 70B model for example, if we double the context length, we can expect the loss to be reduced by a factor of $2^{-\beta} \approx 0.7$ plus a model specific constant $(1 - 2^{-\beta}) \cdot \gamma$. 
    \item Larger models can leverage the contexts more effectively, indicated by the larger $\beta$ value of the curves. 
\end{itemize}

\subsection{Instruction Tuning Results}

We test our instruction tuned model on ZeroSCROLLS~\citep{zeroscrolls} which bundles 10 long-context datasets spanning from summarization, question answering, to multi-document aggregation tasks. For a fair comparison, we use the same configuration (prompts, truncation strategy, and maximum generation lengths, etc) as specified by the benchmark. As shown in Table~\ref{table:zeroscrolls}, without using any human annotated long context data, 
our 70B chat model is able to outperform \texttt{gpt-3.5-turbo-16k} on 7 out of the 10 tasks. In addition, we run evaluations on six new long tasks introduced in L-Eval~\citep{leval} and again observe strong results, as shown in Table~\ref{table:long_eval_results} in the Appendix. We see that the finetuned model is particularly good at QA tasks which is the main theme of the self-instruct data. We expect the performance to be further improved if more diverse data are used for finetuning.

It is worth mentioning that evaluating long-context LLMs is a nontrivial task. The automatic metrics used in these benchmarks are limited in many ways. For instance, the summarization tasks only come with a single ground-truth summary and the $n$-gram matching metrics do not necessarily align with human preference. For QA and aggregation tasks, where the metric is less of a concern, truncating the input context might also remove the information necessary to answer the question. Another important caveat is that most proprietary models do not share their training data details, which makes it hard to take into consideration the potential leakage during public benchmark evaluation. 

\input{tables/zeroscrolls}




\subsection{Human Evaluation}

\begin{figure}[ht]
    \centering
    \vspace{-20pt}
    \includegraphics[width=0.7\linewidth]{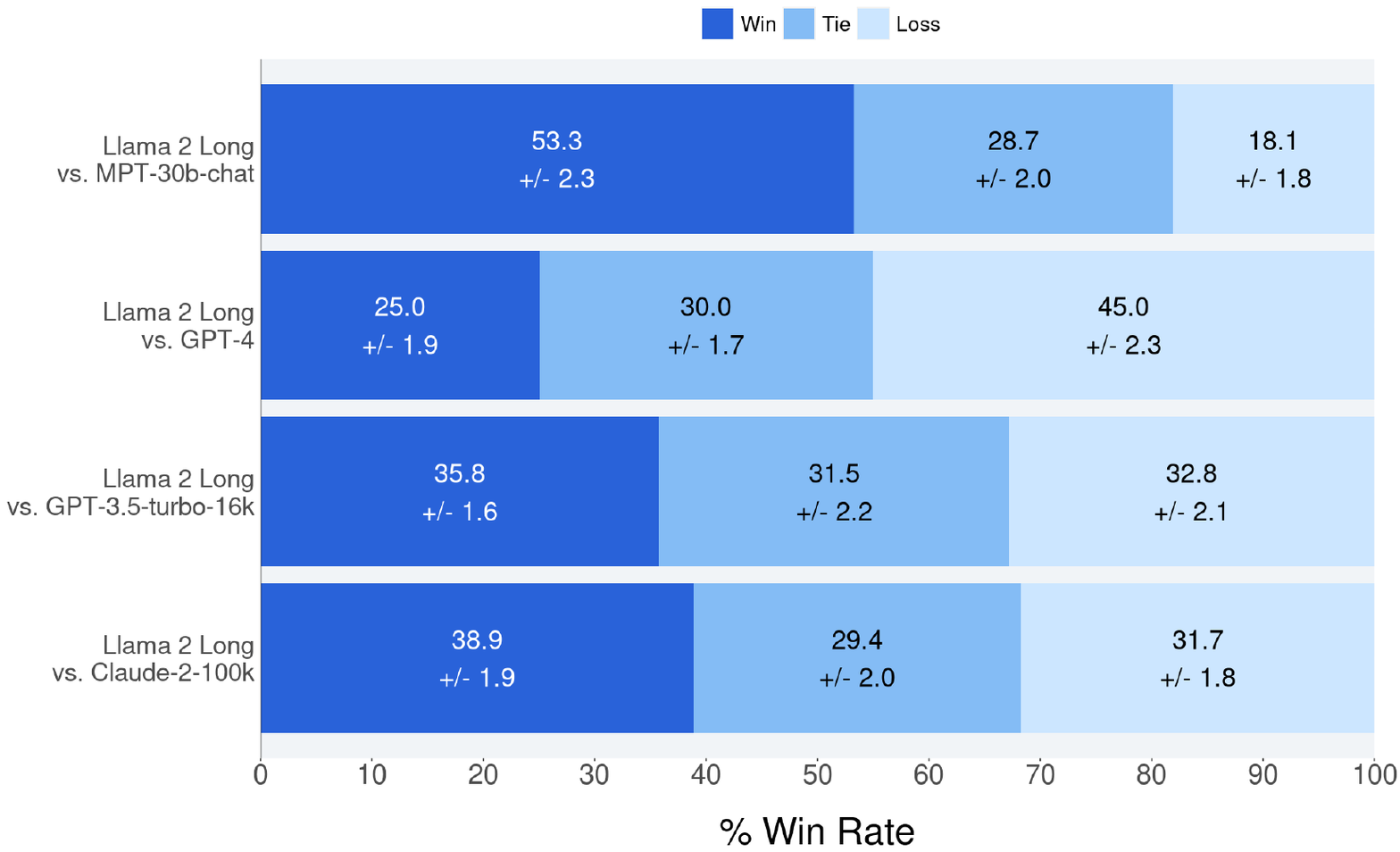}
    \vspace{-20pt}
    \caption{Human preference on model responses with multi-turn conversation and multi-document search query answering data. }
    \label{fig:winrate}
\end{figure}

Complementary to the automatic evaluation benchmark results, we conduct human evaluations by asking annotators whether they prefer the generation from our instruction finetuned model or from proprietary models like MPT-30B-chat, GPT-4, GPT-3.5-turbo-16k, and Claude-2 in terms of helpfulness, honesty, and harmlessness. Unlike automatic metrics, humans are better at evaluating the quality of model responses for long context models because of the large space of acceptable answers. We focus on two major application scenarios with a total of 2,352 examples. For multi-turn conversation data, each prompt is a chat history based on which the model needs to generate a coherent response. For the multi-document search query answering application, the model is provided with a few most relevant documents retrieved from a search session and the corresponding search query. We then evaluate how well these models can leverage the information (retrieved documents) to answer the given query. Each comparison example was evaluated by 3 different human annotators. The standard win rate of our our model over each model is calculated by averaging the result of each comparison example and the final score along with the 95\% confidence interval is shown in Figure~\ref{fig:winrate}. With very little instruction data, our model can achieve competitive performance against MPT-30B-chat, GPT-3.5-turbo-16k, and Claude-2. It is worth noting that human evaluation on longer context tasks is challenging and generally requires well trained and skilled annotators. We hope this study can not only give a sense of the potential of our instruction finetuned model on some long context downstream applications but also motivate future efforts in developing more robust long context automatic evaluations.

%% file: tables/regular_eval.tex
\begin{table}
  
  \centering
  \begin{tabular}{lcccccc}
    \toprule
    \textbf{Model} & Size & Coding & Math & MMLU & Commonsense & OpenQA \\
    \midrule
    \multirow{4}{*}{\llamavtwo} & 7B & 16.8 & 8.55 & 45.3 & 63.9 & 48.9 \\
     & 13B & 24.5 & 16.3 & 54.8 & 66.9 & 55.4\\
     & 34B & 27.8 & 24.2 & 62.6 & 69.9 & 58.7 \\
     & 70B & 37.4 & 35.2 & 68.9 & 71.9 & 63.6 \\
    \midrule
    \multirow{4}{*}{\model} & 7B & 20.6 & 10.5 & 47.8 & 64.9 & 51.0 \\
     & 13B &  25.7  & 21.5 &  60.1  &  67.8 & 56.8 \\
     & 34B & 29.9 & 29.0 & 65.0 & 70.9 & 60.3\\
     & 70B & \textbf{39.9} & \textbf{41.3} & \textbf{71.7} & \textbf{72.7} & \textbf{64.0} \\
    \bottomrule
  \end{tabular}
  \vspace{5pt}
  \caption{Performance on standard short-context benchmarks. We report \textit{Coding} score as the average of pass@1 of HumanEval~\citep{humaneval} and MBPP~\citep{mbpp}; Math score as the average of top-1 accuracy of 8-shot GSM8K~\citep{gsm8k} and 4-shot MATH~\citep{math}; OpenQA score as the average of 5-shot performance on NaturalQuestions~\citep{natural_questions} and TriviaQA~\citep{triviaqa}; \textit{Commonsense} score as the average of PIQA~\citep{piqa}, SIQA~\citep{siqa}, HellaSwag~\citep{hellaswag}, WinoGrande~\citep{winogrande}, ARC easy and challenge~\citep{arc}, OpenBookQA~\citep{OpenbookQA} and CommonsenseQA~\citep{commonsenseqa}.}
  \label{table:pretrain_regular_eval}
\end{table}

%% file: tables/compare_gpt_3.5.tex
\begin{table}[t]
\small
\centering
\begin{adjustbox}{max width=1.2\linewidth, center}
\begin{tabular}{lcccccc}
    \toprule
    \textbf{Task} & GPT-3.5 & GPT-4 & PaLM & PaLM-2-L & \llamavtwo & \model \\
    \midrule
    MMLU (5-shot) & 70.0 & \textbf{86.4} & 69.3 & 78.3 & 68.9 & 71.7 \\
    Natural Questions (1-shot) & - & - & 29.3 & \textbf{37.5} & 33.0 & 35.7 \\
    GSM8K (8-shot) & 57.1 & \textbf{92.0} & 56.5 & 80.7 & 56.8 & 65.4 \\
    HumanEval (0-shot) & 48.1 & \textbf{67.0} & 26.2 & - & 29.9 & 32.9 \\
    \bottomrule
\end{tabular}
\end{adjustbox}
\vspace{5pt}
\caption{Comparison with closed models on standard short tasks.}
\label{table:compare_openai}
\end{table}

%% file: tables/compare_with_oss.tex
\begin{table}[h!]
  \small
  \centering
  \begin{adjustbox}{max width=1.2\linewidth, center}
      \begin{tabular}{lcccc}
        \toprule
        \multirow{2}{*}{\textbf{Model}}  & NarrativeQA  & Qasper  & QuALITY & QMSum \\
        & F1 (0-shot) & F1 (2-shot) & EM (2-shot) & ROUGE-geo$^*$ (1-shot) \\
        \midrule
        Focused Transformer (3B)   & 16.3 & 15.4 & 20.5 & 10.6\\
        Yarn-7B-128k  & 20.9 & 26.2 & 32.3 & 11.4 \\
        Together-7B-32k$^\dagger$  & 23.3 & 27.3 & 41.2 & 12.6\\
        Xgen-7B-8k-base  & 17.4 & 20.5 & 21.0 & 6.79  \\
        MPT-7B-8k        & 18.8 & 24.7 & 23.7 & 8.78  \\
        Yarn-13B-128k & 23.4 & 27.1 & 46.4 & 11.9 \\
        MPT-30B-8k       & 22.9 & 29.0 & 41.5 & 10.3 \\
        \midrule
        \llamavtwo 70B & 25.7 & 27.5 & 53.0 & 11.9 \\
        \midrule
        \model 7B   & 21.9 & 27.8 & 43.2 & 14.9 \\
        \model 13B  & 25.6 & 31.2 & 57.6 & 15.7 \\ 
        \model 34B  & 29.4 & 33.7 & 65.7 & 15.9 \\ 
        \model 70B & \textbf{30.9} & \textbf{35.7} & \textbf{79.7} & \textbf{16.5} \\
        \bottomrule
      \end{tabular}
  \end{adjustbox}
  \vspace{5pt}
  \caption{Comparison with open-source long-context models on research benchmarks. $^\dagger$: ``together-7B-32k'' is not a purely pretrained model and has been trained using supervised datasets which can improve its few-shot results. $^*$: ROUGE-geo is the geometric mean of ROUGE-1, 2 and L.
  All numbers are validation results and the maximum allowed prompt length is set to 16,384 tokens.
  }
  \label{table:compare_with_oss}
\end{table}

%% file: tables/zeroscrolls.tex
\begin{table*}[ht]
    \small
    \centering
    \begin{adjustbox}{max width=1.2\linewidth, center}
    \begin{tabular}{l|ccc|ccccc|cc|c}
        \toprule
        \multirow{2}{*}{\textbf{Model}} &  \multicolumn{3}{c|}{Summarization} & \multicolumn{5}{c|}{Question answering} & \multicolumn{2}{c|}{Aggregation}  \\
        & GR & SS & QM & SQAL & Qspr & Nrtv & QALT & MuSQ & SpDg & BkSS & Avg \\
        \midrule
        
        GPT-3.5-turbo (4k)       & 21.3 & 16.1 & 15.6 & 20.4 & 49.3 & 25.1 & 66.6 & 27.1 & 49.1 & 49.8 & 34.0 \\
        GPT-3.5-turbo-16k$^{\dagger}$    & 24.3 & 16.2 & 17.4 & 21.4 & 50.0 & 29.5 & 72.0 & 27.0 & 54.1 & 54.6 & 36.7 \\
        Claude (8k) & 24.2 & 16.1 & 14.6 & 21.0 & 52.3 & 32.6 & 84.8 & 36.1 & 61.6 & 47.4 & 39.1 \\ 
        GPT4 (8k)        & 26.3 & 17.3 & 18.5 & 22.6 & 50.7 & 27.6 & 89.2 & 41.1 & 62.8 & 60.5 & 41.7 \\
        \midrule
  
        \modelchat  70B  & \underline{26.0} & 15.0 & \underline{20.0} & 20.9 & \underline{52.0} & \underline{31.7} & \underline{82.6} & \underline{27.3} & \underline{55.5} & 46.2 & 37.7 \\
        \bottomrule
    \end{tabular}
    \end{adjustbox}
    \caption{ZeroSCROLLS long-context leaderboard results.  $^{\dagger}$Evaluated as of 8/7/2023. The GPT-4 and Claude results are directly copied from the leaderboard. Underscored are the 7/10 tasks where our model outperforms \texttt{gpt-3.5-turbo-16k}. 
    }
    \label{table:zeroscrolls}
\end{table*}

%% file: sections/pe.tex
\subsection{Positional Encoding for Long Text}
\label{pe_analysis}


\begin{wraptable}{r}{0.45\textwidth}
  \small
  \centering
  \begin{tabular}{lccc}
    \toprule
    \textbf{PE} &  Books & CC & Wikipedia \\
    \midrule
    RoPE & 6.548 & 6.816 & 3.802\\
    \midrule
    \textsc{RoPE PI} & 6.341 & 6.786 & 3.775\\
    \textsc{RoPE ABF} & \textbf{6.323} & \textbf{6.780} & \textbf{3.771} \\
    \textsc{xPos ABF} & 6.331 & \textbf{6.780} & \textbf{3.771} \\
    \bottomrule
  \end{tabular}
  \caption{Validation perplexity of models with different positional encoding variants. All samples are $32,768$-token sequences (CC: CommonCrawl).}
  \label{table:pe_ablation_ppl}
\end{wraptable}

Our early experiments used a synthetic ``\textsc{first-sentence-retrieval}'' task to probe the effective context window of the pretrained models where we simply prompt the model to return the first sentence of the input. Our initial task results suggest that, with the original \llamavtwo architecture untouched, our model was unable to effectively attend beyond $4,000$ - $6,000$ tokens even after extensive long-context continual pretraining. We hypothesize that this bottleneck comes from the \textsc{RoPE} positional encoding used in \llamavtwo series which imposes a heavy decay on the attention scores\footnote{The quantity that heavily decays is $\mathbb{E}_{q, k}[\textsc{RoPE}(q, m)^\top\textsc{RoPE}(k, n) | m, n]$ as the relative position $|m - n|$ gets larger where $q, k$ are the query and key of the two tokens at position $m$ and $n$.} for distant tokens. We propose a simple modification to the default RoPE encoding to reduce the decaying effect -- increasing the ``base frequency $b$'' of \textsc{RoPE} from $10,000$ to $500,000$, which essentially reduces the rotation angles of each dimension. The idea is also concurrently suggested in the Reddit \cite{ntk_link} community and ~\citet{codellama}. The effect of the base frequency change is visualized in Figure~\ref{fig:pe_decay}.
Another concurrent approach named ``position interpolation'' (PI) \citep{pi} proposes to linearly scale the input positions such that the positions of tokens in the long sequences will to mapped to the model's original position range. As shown by the figure, it also implicitly achieves a decay reduction effect. 

Another interesting observation from the visualization is that 
RoPE introduces large ``oscillation'' in the long-range regions, which could be undesirable for language modeling~\citep{xpos}. To investigate whether this effect hurts performance, we also explored another recently proposed variant of rotary encoding, \textsc{xPos}~\citep{xpos}, which smooths the high-frequency component. Note that \textsc{xPos} with the default parameters suffers from the same decaying issue as \textsc{RoPE} and therefore, we also applied a similar decay fix to \textsc{xPos}. 

Specifically, we empirically compare the following methods: the \textsc{RoPE} baseline, \textsc{PI}, our proposed RoPE with adjusted base frequency (denoted as \textsc{RoPE ABF}), and \textsc{xPos ABF} (visual comparisons in Figure~\ref{fig:pe_decay}). We report results on 1) long-sequence validation perplexity in Table~\ref{table:pe_ablation_ppl} and Figure~\ref{fig:pe_ablation_ppl_evo}, 2) the \textsc{first-sentence-retrieval} context probing task\footnote{We also test on the \textsc{PassKey} task as used in~\citep{passkey}. All the model variants except \textsc{RoPE} can achieve perfect accuracy. We believe this task is overly simple for context probing.} in Figure~\ref{fig:pe_ablation_synthetic}, and 3) some representative regular context tasks in Table~\ref{table:pe_ablation_short} (to validate that long models do not degenerate on short-context tasks). 
All model variants are continually pretrained from the 7B \llamavtwo checkpoint with additional 80B tokens organized as 32,768-token long sequences.

\begin{figure}[ht]
    \centering
    \includegraphics[width=0.6\linewidth]{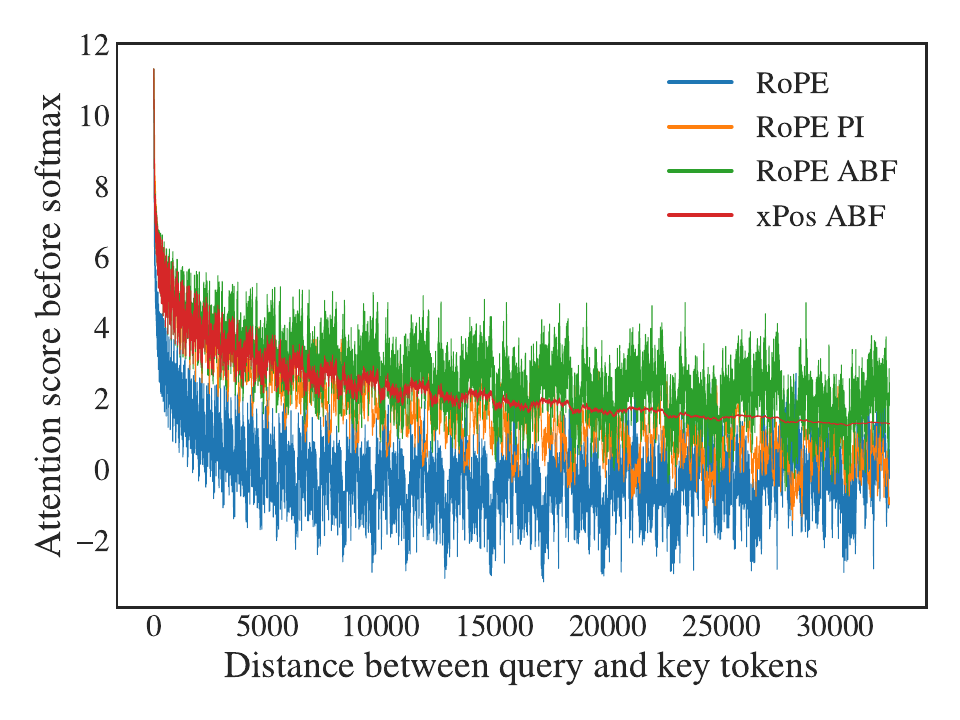}
    \caption{Decaying raw attention scores for distant tokens of explored positional encoding variants (assuming keys and queries are all-ones vectors).}
    \label{fig:pe_decay}
\end{figure}


Overall, results on these evaluations suggest that \textsc{RoPE ABF} performs the best among all explored variants. In particular, we see that \textsc{RoPE ABF} is the only variant that can maintain its performance up to the full 32,768-token context window on the \textsc{first-sentence-retrieval} task. We also found that \textsc{xPos ABF} with less oscillation does not lead to substantial gains, suggesting that these artifacts are not detrimental to language modeling. 
While \textsc{xPos} is claimed to possess better extrapolation property \citep{xpos}, we found that, with the base frequency modification, \textsc{xPos} does not extrapolate better than \textsc{RoPE}
(see Appendix \ref{app:extrapolation}). In addition to empirical results, we provide a theoretical analysis of RoPE ABF and its difference to PI in Appendix \ref{sec:app_detailed_results}. We argue that RoPE ABF distributes the embedded vectors with an increased granularity when compared to RoPE PI, making it a easier for the model to distinguish between positions.
It is worth noting that the relative distance between the embedded vectors has a linear dependence on the key parameter of RoPE PI and a logarithmic dependence on the key parameter of RoPE ABF, which coincides with our empirical observation that the base-frequency is not very sensitive and can be easily adjusted based on the max sequence length.


\begin{table}
  \centering
  \small
  \begin{tabular}{lccccc}
    \toprule
    & HumanEval & Math & MMLU & HellaSwag & TQA \\
    \midrule
    RoPE & 14.63 & \textbf{3.62} & 45.69 & 76.31 & 65.23 \\
    \midrule
    \textsc{RoPE PI} & 15.24 & 3.08 & 45.84 & 76.65 & 65.96 \\
    \textsc{RoPE-ABF} & \textbf{17.07} & 3.52 & \textbf{46.24} & \textbf{76.73} & 66.04 \\
    \textsc{xPos-ABF} & 16.46 & 3.54 & 45.72 & 76.68 & \textbf{66.14} \\
    \bottomrule
  \end{tabular}
  \vspace{5pt}
  \caption{The performance of models with different positional encoding variants on standard short-context benchmarks.}
  \label{table:pe_ablation_short}
\end{table}


\begin{figure}
\begin{adjustbox}{max width=1.0\linewidth, center}
    \begin{subfigure}{.7\linewidth}
      \centering
      \includegraphics[width=\linewidth]{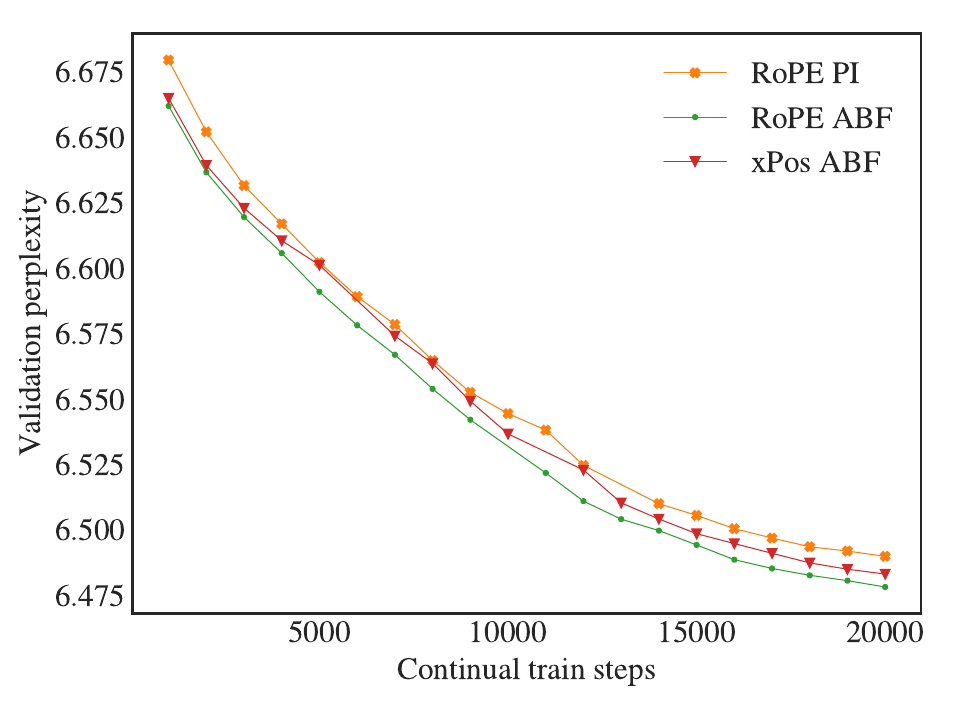}
      \caption{Validation PPL (16k-token sequences) on a held-out long-context dataset.}
      \label{fig:pe_ablation_ppl_evo}
    \end{subfigure}
    \hspace{-10pt}
    \begin{subfigure}{.7\linewidth}
      \centering
      \includegraphics[width=\linewidth]{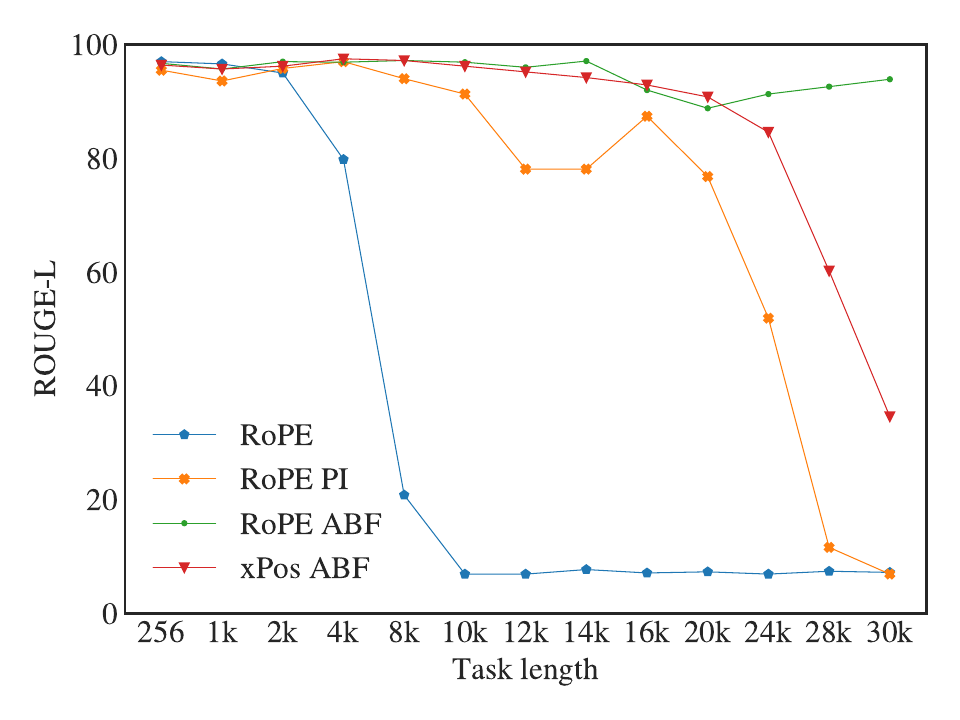}
      \caption{Performance on \textsc{first-sentence-retrieval} task.}
      \label{fig:pe_ablation_synthetic}
    \end{subfigure}
\end{adjustbox}
\caption{Comparison of positional encoding variants on synthetic sentence retrieval task and validation perplexity evolution during continual pretraining. }
\end{figure}

%% file: sections/pretrain_data.tex
\subsection{Pretraining Data Mix}
\label{datamix_analysis}

The data used to continually pretrain our model combines existing datasets used by \llamavtwo and new long text data. We also adjusted the data source mix ratio to up-weight long data samples. Our early experiments with 7B models confirms the significant improvements using this data mix for long-context tasks, as shown in the first two rows of Table~\ref{table:data_ablation_scrolls}.
In this section, we aim to rigorously investigate the source of improvements. In particular, we are interested in differentiating the effects of the data length distribution and the quality of the corpus itself. 

We perform two additional ablations using \llamavtwo's pretrain datasets: 1) we remove the long text data from the \llamavtwo dataset and continually pretrain our model with mostly short documents; 2) we increase the sample weights of existing long text data to be similar to the long text ratio used by proposed new model. Interestingly, even with most of the long texts removed, the model can still obtain most of the performance gain over \llamavtwo. We also find that there is no clear and consistent advantage as we greatly increase the long data ratio (the third row v.s. the fourth row in Table~\ref{table:data_ablation_scrolls} and Table~\ref{table:data_ablation_short}). 
We observe similar results on the \textsc{first-sentence-retrieval} task as shown by Figure~\ref{fig:data_ablation_synthetic} in the Appendix.

Based on the above ablations, we can see that adjusting the length distribution of the pretrain data does not provide major benefits. However, as we evaluate these model variants' performance on standard short-context tasks, we find that new data mix also leads to large improvements in many cases, especially knowledge-intensive tasks like MMLU, as shown in Table~\ref{table:data_ablation_short}. These results suggest that \textit{long-context LLMs can be effectively trained even with very limited long data} and the improvements of our pretrain data over the one used by \llamavtwo mostly come from the quality of the data itself, instead of the length distribution difference.

\begin{table}[t]
  \small
  \centering
  \begin{tabular}{lcccc}
    \toprule
    \multirow{2}{*}{\textbf{Continual Pretrain Data}} & NarrativeQA & Qasper & Quality & QMSum \\
    & $\Delta$ F1 & $\Delta$ F1 & $\Delta$ EM & $\Delta$ ROUGE-geo \\
    \midrule
    \model data mix & \textbf{23.70\%} & \textbf{43.64\%} & \textbf{75.5\%} & \textbf{45.70\%} \\
    \llamavtwo data mix        & 18.23\% & 38.12\% & 60.3\% & 44.87\% \\
    - remove long text data          & 19.48\% & 39.14\% & 67.1\% & 36.60\% \\
    - upsample existing long text data & 22.15\% & 36.82\% & 65.0\% & 42.83\% \\
    
    \bottomrule
  \end{tabular}
  \vspace{5pt}
  \caption{Comparison of different pretraining data mix on long-context tasks. Instead of showing the absolute performance, we report relative improvements over the 7B \llamavtwo which has a 4,096-token context window. All models are evaluated with prompts truncated at 16,384 tokens.}
  \label{table:data_ablation_scrolls}
\end{table}


\begin{table}[h]
  \small
  \centering
  \begin{adjustbox}{max width=1.2\linewidth, center}
  \begin{tabular}{lccccc}
    \toprule
    \textbf{Continual Pretrain Data} & HumanEval & Math & MMLU & HellaSwag & TQA \\
    \midrule
    \model data mix & \textbf{17.08} & \textbf{4.09} & \textbf{48.62} & 76.74 & 66.24 \\
    \llamavtwo data mix             & 15.24 & 3.61 & 46.30 & 76.63 & \textbf{66.71} \\
    - remove long text data               & 17.07 & 3.57 & 46.25 & \textbf{76.76} & 65.90 \\
    - upsample existing long text data & 17.07 & 3.53 & 46.25 & 76.74 & 66.04 \\

    \bottomrule
  \end{tabular}
  \end{adjustbox}
  \vspace{5pt}
  \caption{Standard short task performance of long-context models with different pretrain data mix.}
  \label{table:data_ablation_short}
\end{table}

%% file: sections/sft.tex
\subsection{Instruction Tuning}
\label{sec:sft}

We explored various strategies to instruction-finetune the pre-trained long context model which do not require any supervised long data. We start with only finetuning the models with short instruction data from \chatllama (referred as "RLHF V5" in \citep{llama2}) and then blend in with some pretrain data to avoid forgetting of previous long context continual pretraining. As demonstrated in Table~\ref{table:instruct_tuning_ablation}, using only short instruction data can already produce a decent long model that significantly outperforms \llamavtwo on various long-context tasks. On top of this dataset that only includes short prompts, we see that adding pretrain data (calculating language modeling loss on the whole sequence) can further boost the performance on most datasets. 
Inspired by this, we add the LM loss over the long context inputs when we finetune with self-instruct data. This simple trick makes learning more stable when we have unbalanced input and output lengths\footnote{In our cases, the output lengths of most samples are a lot shorter than the those of the long-context inputs.}, which gives significant improvements on most of the tested tasks (the last two rows of Table~\ref{table:instruct_tuning_ablation}).




\begin{table*}[ht]
  
  \small
  \centering
  \begin{adjustbox}{max width=1.2\linewidth, center}
  \begin{tabular}{l|ccccc}
    \toprule
    \textbf{Settings} & Qasper & NarrativeQA & QuALITY & SummScreenFD & QMSum \\
    \midrule
    \chatllama baseline & 12.2 & 9.13 & 56.7 & 10.5 & 14.4 \\
    \midrule
    \multicolumn{2}{l}{\model \textit{finetuned} with:} \\
    
    "RLHF V5"    & 22.3 & 13.2 & 71.4 & 14.8 & 16.9 \\
    "RLHF V5" mix pretrain  & 23.7  & 16.6 & 76.2 & \textbf{15.7} & 17.8 \\
    "RLHF V5" mix self-inst w/o LM loss & 35.7 & 22.3 & 59.3 & 12.2 & 13.4 \\
    "RLHF V5" mix self-inst with LM loss & \textbf{38.9} & \textbf{23.3} & \textbf{77.3} & 14.5 & \textbf{18.5}\\
    \bottomrule
  \end{tabular}
  \end{adjustbox}
  \caption{Comparison of different instruction finetuning data mixes.}
  \label{table:instruct_tuning_ablation}
\end{table*}

%% file: sections/curriculum.tex
\subsection{Training Curriculum}
\label{sec:curriculum_analysis}




Continual pretraining has demonstrated its efficacy in our experiments, but an open question still remains: does pretraining from scratch with long sequences yield better performance than continual pretraining? In this section, we study different training curricula and try to investigate if continual pretraining can offer competitive performance with less computation budget. We start off by pretraining a 7B model with 32,768 sequence length from start to the end. Then we explored various two-stage training curricula where we begin with 4096 sequence length and switch to 32,768 when the model completes 20\%, 40\%, 80\% of whole training process. For all cases, we keep the same number of total training tokens and make sure the number of tokens per each gradient update remains constant (4 million tokens) by adjusting the batch size and sequence length accordingly. 


We evaluate our models on the long-text QA tasks used in Section~\ref{datamix_analysis} and report the final models' perplexity on different validation corpora. 
As shown in Table~\ref{table:scrolls_curriculum} and Table~\ref{table:ppl_curriculum}, continual pretraining from short context models can easily save around 40\% FLOPs while imposing almost no loss on performance. These results also align with the training loss curves we observed from each run in Figure~\ref{fig:training_curve_curriculum_smooth_1} -- the models can quickly adapt to the increased sequence length and get to similar loss scale.

\begin{table}[h]
  \centering
  \begin{tabular}{lccccc}
    \toprule
    \multirow{2}{*}{\textbf{Pretrain Curriculum}} & \multirow{2}{*}{\textbf{FLOPs}} & NarrativeQA & Qasper & Quality & QMSum \\
    & &  F1 &  F1 &  EM & ROUGE-geo \\
    \midrule
    32k from scratch & $3.783\times 10^{22}$ & 18.5 & 28.6 & 37.9 & 11.46 \\
    4k$\rightarrow$32k $@$ 20\%  & $3.405 \times 10^{22}$ & 20.0 & 28.1 & 38.8 &  12.09\\
    4k$\rightarrow$32k $@$ 40\% & $3.026 \times 10^{22}$ & 20.1 & 27.0 & 37.4 & 12.44\\
    4k$\rightarrow$32k $@$ 80\% & $2.270 \times 10^{22}$ & 18.5 & 25.0 & 38.3 & 11.00 \\
    \bottomrule
  \end{tabular}
  \vspace{5pt}
  \caption{Comparison of models with different training curricula on long context QA tasks.}
  
  \label{table:scrolls_curriculum}
\end{table}

\begin{table}[h]
  \centering
  \begin{tabular}{lcccc}
    \toprule
    \textbf{Model} & CC & Books & Wikipedia \\
    \midrule 
    32k from scratch       & 7.67 & 6.52 & 4.31 \\
    4k$\rightarrow$32k $@$ 20\%  & 7.59 & 6.46 & 4.26 \\
    4k$\rightarrow$32k $@$ 40\% & 7.59  & 6.46 & 4.25 \\
    4k$\rightarrow$32k $@$ 80\% & 7.59  & 6.49 & 4.25 \\
    \bottomrule
  \end{tabular}
  \vspace{5pt}
  \caption{Perplexity evaluation of models with different training curricula on three validation sets. }
\label{table:ppl_curriculum}
\end{table}

\begin{figure}[h!]
\begin{adjustbox}{max width=1.0\linewidth, center}
    \begin{subfigure}{.7\linewidth}
      \centering
      \includegraphics[width=\linewidth]{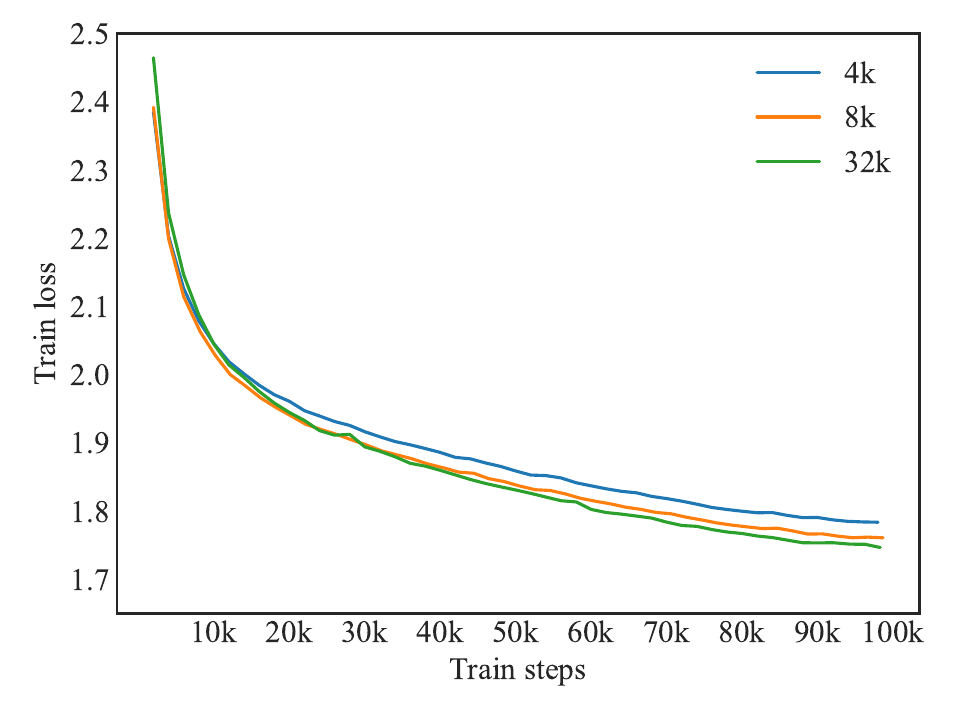}
    \end{subfigure}
    \hspace{-10pt}
    \begin{subfigure}{.7\linewidth}
      \centering
      \includegraphics[width=\linewidth]{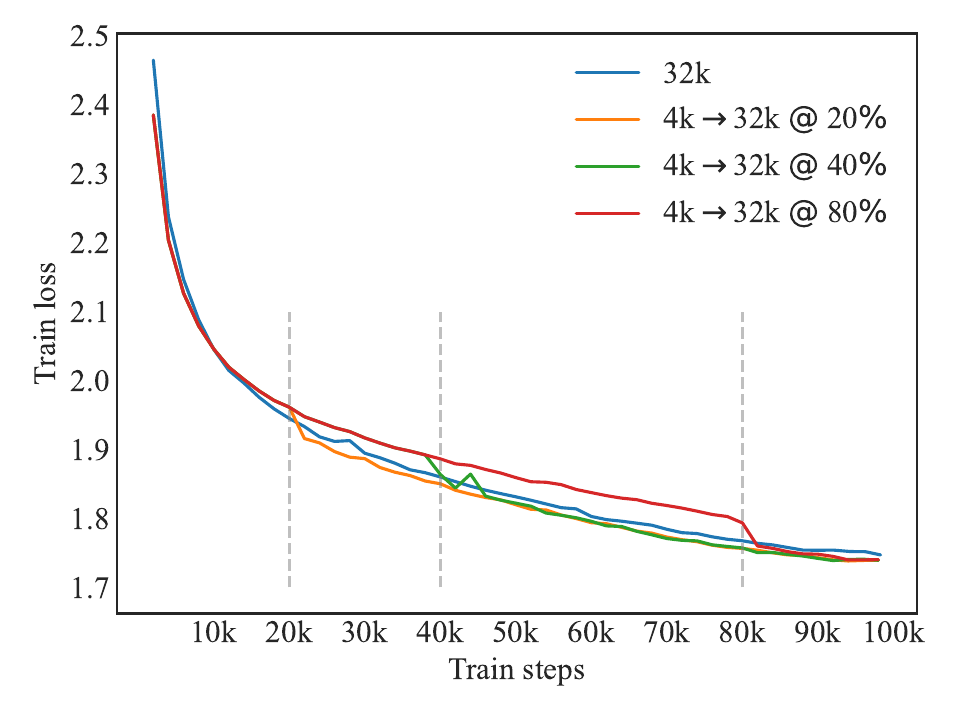}
    \end{subfigure}
    \hspace{-10pt}
\end{adjustbox}
\caption{Smoothed loss curves for the training curriculum ablation. On the left, we show losses for models trained with a fixed context window. On the right, we compare training curricula where we switch the context length from 4,096 to 32,768 tokens at different stages indicated by the dashed lines. Our models can quickly adapt to the new sequence length within a few thousand steps. }
\label{fig:training_curve_curriculum_smooth_1}
\end{figure}





%% file: sections/safety.tex
\subsection{Evaluation on Safety Benchmarks}
Despite showing excellent performance on various of downstream tasks, large language models are prone to generating harmful, misinformative, and biased contents \citep{lin2021truthfulqa, hartvigsen2022toxigen, dhamala2021bold, ji2023survey}. 
Long-context language models can process extended inputs in their context window, but at the same time, they also face a higher risk of jailbreak, especially through means such as prompt injection \citep{greshake2023not}. 
In this section, we evaluate the safety capability of instruction fine-tuned model using three standard academic benchmarks including TruthfulQA \citep{lin2021truthfulqa}, ToxiGen \citep{hartvigsen2022toxigen}, and BOLD \citep{dhamala2021bold}, similar to \citep{llama2}. 
We focus on the largest instruction fine-tuned model variant (i.e., 70B) and compare its results with both open sourced LLMs (Falcon-instruct \cite{falcon40b}, MPT-instruct \cite{MPT30b}) and propriety LLMS (GPT-3.5, GPT-4 \citep{gpt-4}, Claude-2 \citep{claude}) in Table~\ref{tab:safety_evaluation}.

We observe that in general instruction fine-tuned model maintains similar safety performance compared to \chatllama and is safer and less biased compared to other open-source LLMs such as Falcon-instruct and MPT-instruct.
AI safety is a complex domain and it can be extremely difficult to comprehensively evaluate all safety aspects of instruction fine-tuned model with three benchmarks. 
However, we hope our analysis can serve as a pilot study and provide directional signals on long-context large language models' safety performance, which are not discussed in other works on the same topic \citep{long_llama, long_net, pi}.
Currently the community also lacks dedicated safety benchmarks for long-context large language model evaluation and we plan to invest in this direction in our future work. 

\paragraph{TruthfulQA} We evaluate instruction fine-tuned model on TruthfulQA \citep{lin2021truthfulqa} to benchmark its factuality. The benchmark consists of 817 questions covering 38 categories including health, law, finance, and politics \citep{lin2021truthfulqa}. 
Similar to \citep{llama2}, we use few-shot prompts with 6 random QA pairs for generation and then leverage two fine-tuned GPT-3 models to classify whether the generation is truthful and informative. 
We report the percentage of generations that are both truthful and informative as the final metric in Table~\ref{tab:safety_evaluation}.

\paragraph{ToxiGen} We measure the toxicity of instruction fine-tuned model using ToxiGen \citep{hartvigsen2022toxigen} where we check the percentage of toxic and hateful generations against 13 minority groups. 
Following \citep{llama2}, we filtered out prompts where annotators disagree with each other on the target demographic group. 
We use the default ToxiGen classifier fine-tuned based on RoBERTa \citep{liu2019roberta} to evaluate the level of toxicity of the model's outputs. 
We report the percentage of toxic generations across all groups in Table~\ref{tab:safety_evaluation}.

\paragraph{BOLD} Bias in Open-Ended Language Dataset (BOLD) \cite{dhamala2021bold} is used in this work to quantify how biased the models are against people from different demographic groups. 
This dataset consists of 23,679 prompts extracted from English Wikipedia covering five domains including race, gender, religion, political ideology and profession with 43 subgroups in total.
Following \cite{llama2}, we exclude prompts belonging to Hinduism and Atheism religious subgroups as they only feature 12 and 29 prompts, respectively.
After generations are inferred from each model, we leverage the Valence Aware Dictionary and Sentiment Reasoner (VADER) \cite{hutto2014vader} to perform sentiment analysis with a score ranging between -1 and 1. 
A positive score corresponds to a positive sentiment towards the subgroup mentioned in the prompt and vice versa. 
A sentiment score close to 0 indicates neutral sentiment which is desired.
We report the average sentiment score across 43 demographic subgroups as the final metric for BOLD in Table~\ref{tab:safety_evaluation}. 

\begin{table}[htbp]
\centering
\begin{tabular}{lrccc}
\toprule
& Model Size & {TruthfulQA $\uparrow$} & {ToxiGen $\downarrow$} & {BOLD $\downarrow$}\\
\midrule
GPT-3.5-turbo  & - & 78.46 & 0.01 & 0.50 \\ 
GPT-3.5-turbo-16k & - & 75.15 & 0.07 & 0.49 \\ 
Claude-2 & - & 62.66 & 0.05 & 0.46 \\ 
GPT4 & - & \textbf{80.66} & 0.03 & 0.43 \\ 
Falcon-instruct & 40B & 57.41 & 3.3 & 0.39 \\ 
MPT-instruct & 30B & 42.71 & 16.85 & \textbf{0.34} \\ 
\midrule  
\chatllama & 70B & 64.14 & 0.01 & 0.41 \\ 
\midrule
\modelchat & 70B & 60.95 & \textbf{0.00} & 0.40 \\
\bottomrule
\end{tabular}
\vspace{5pt}
\caption{Evaluation of fine-tuned LLMs on three safety benchmarks. For TruthfulQA, we present the percentage of generations that are both truthful and informative (the higher the better). For ToxiGen, we present the percentage of toxic generations across all groups (the smaller the better). For BOLD, we report the average sentiment score across 43 demographic groups (the closer to 0 the better). 
}
\label{tab:safety_evaluation}
\end{table}

\subsection{Red Teaming Exercises}
Currently there is no open-sourced safety benchmark designed for long-context understanding. 
To ensure that the models are safe in long context use scenarios, we performed internal red teaming to better understand the vulnerability of our chat model. 
We attack the model by feeding long contexts (e.g., long conversations) to it, followed by adversarial prompts covering risky areas including illicit and criminal conducts (e.g., terrorism, theft, and human trafficking), hateful and harmful behaviors (e.g., defamation, self-harm, eating disorders, and discrimination), and unqualified advice \cite{llama2}.
Through manual inspection, we did not observe significant risks compared to \chatllama \cite{llama2}.
We plan to invest more in new attack vectors against long context large models in future work. 

%% file: sections/conclusion.tex

We present a series of long-context LLMs that leverage a simple yet necessary position encoding refinement and continual pretraining to achieve strong long-context performance. Our long context scaling is performed by continually pretraining from \llamavtwo with additional 400B tokens and outperform \llamavtwo on both short and long-context tasks. Our models also demonstrate superior performance compared to existing open-source long-context models and compare favorably against \texttt{gpt-3.5-turbo-16k} on a suite of long-context tasks after a simple instruction finetuning procedure without human supervision. We complement our results with a comprehensive analysis, providing insights on the influences of various factors including the nuances of position encodings, the data mix, and the pretraining curriculum on the final performance. We hope our study could make long-context LLMs more accessible and facilitate further advancements in this field.

%% file: sections/appendix.tex
\clearpage
\appendix

\section{More Results}
\label{sec:more_results}

\input{tables/compare_with_oss_32k}

\input{tables/mmlu_decomposed_results}

\input{tables/more_decomposed_results}

\input{tables/commonsense_decomposed}

\input{tables/longeval_tasks}

\begin{figure}[h!]
    \centering
    \includegraphics[width=0.55\linewidth]{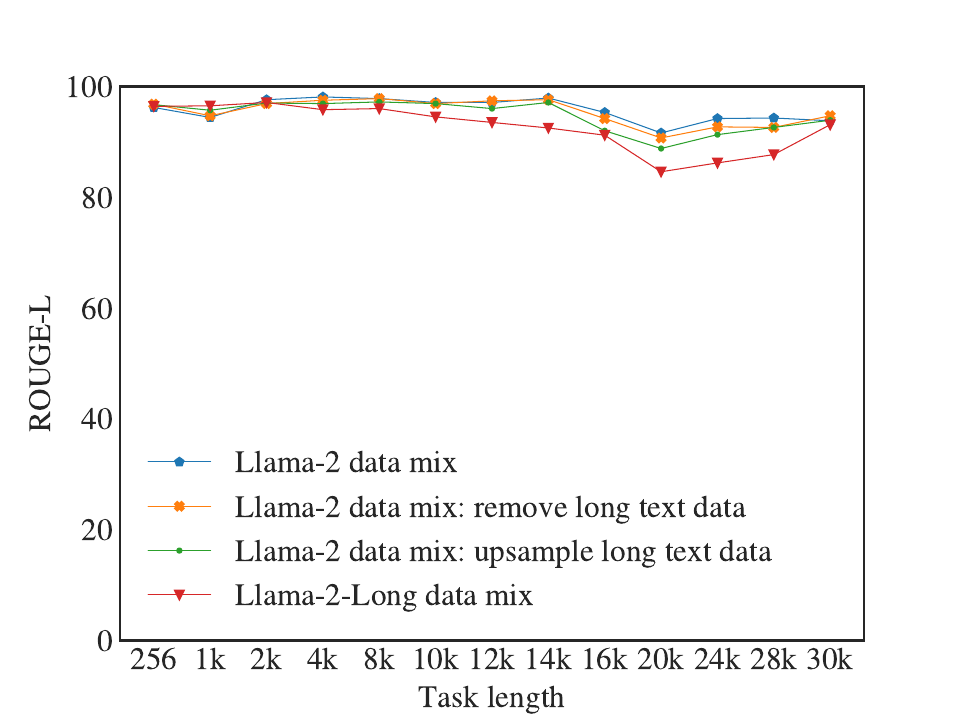}
    \caption{ \textsc{first-sentence-retrieval} performance of models trained with different data mixes.}
    \label{fig:data_ablation_synthetic}
\end{figure}

\section{Theoretical Analysis of Positional Encodings}
\label{sec:app_detailed_results}

\input{sections/pe_theory_appendix}

\begin{figure}[h!]
\begin{adjustbox}{max width=1.0\linewidth, center}
    \begin{subfigure}{.7\linewidth}
        \centering
        \includegraphics[width=\linewidth]{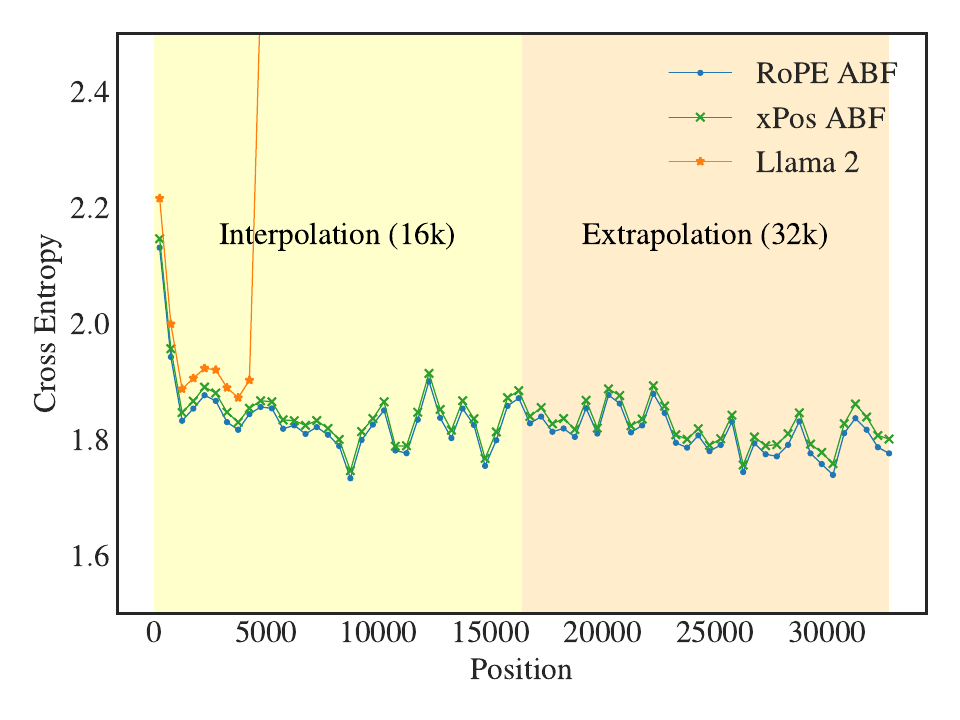}
        \caption{Validation loss calculated at each position of 32,768 context window. }
        \label{fig:loss_at_each_pos}
    \end{subfigure}
    \hspace{-10pt}
    \begin{subfigure}{.7\linewidth}
        \centering
        \includegraphics[width=\linewidth]{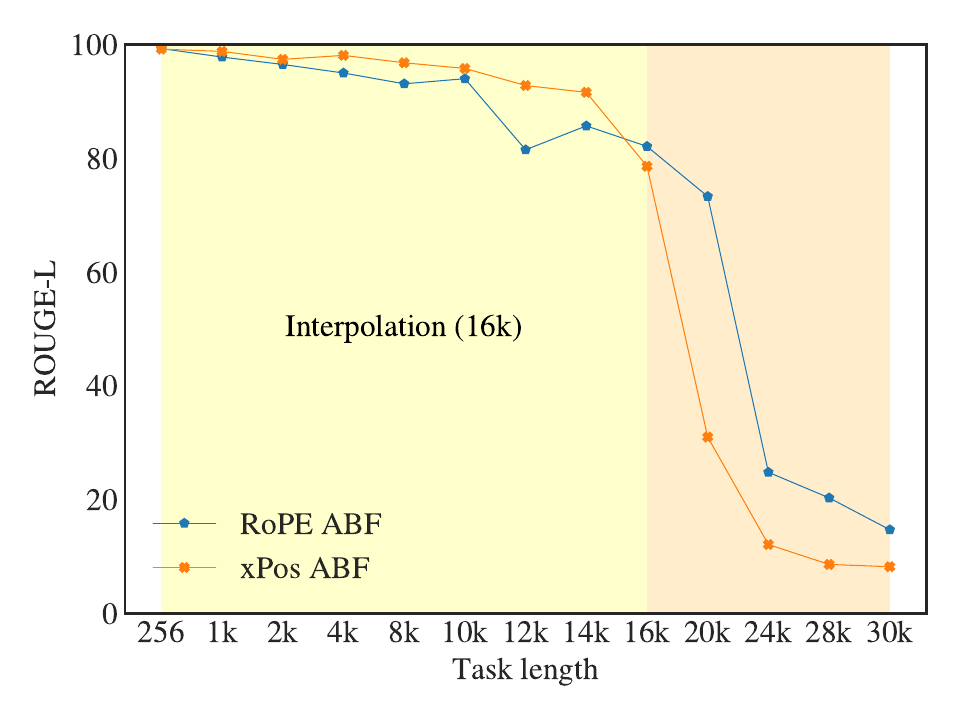}
        \caption{Context window probing with \textsc{first-sentence-retrieval} task. }
        \label{fig:get_1st_70B_extrapolation}
    \end{subfigure}
\end{adjustbox}
\caption{Evaluation on 
our 70B model's extrapolation capabilities. }
\end{figure}

\section{Length Extrapolation Results}
\label{app:extrapolation}
Despite not the focus of this work, extrapolation is an important property for long context models. Extrapolation refers to a model's ability to conduct inference on input sequences that are longer than its training sequences. We evaluate how 
our 70B model extrapolates with two tasks: 

\begin{itemize}
    \item \textbf{Validation loss at each position}: In Figure \ref{fig:loss_at_each_pos}, we visualize the average loss at each position of the 32,768 sequence length where the first 16,384 is the interpolation area (within training sequence length) and the second half is extrapolation. We use 50 batches of samples and average across them. To make plots smoother, we also take the mean of losses every 500 positions. As we can see, our 
    70B model with either \textsc{RoPE ABF} or \textsc{xPos ABF} maintain the loss in the extrapolation area. To contrast this, we also plot the result for \llamavtwo with 4,096 context window: the loss explodes after the position goes beyond training sequence length, which suggests that \llamavtwo does not extrapolate effectively. 

    \item \textbf{Synthetic \textsc{first-sentence-retrieval} task}: To complement validation loss evaluation, we also test 
    our 70B model with two different PEs on the context probing task. Unlike validation loss task where it is hard to find data samples that require very long range dependencies consistently, \textsc{first-sentence-retrieval} imposes a very strict requirement for models to attend with a specific length. In Figure \ref{fig:get_1st_70B_extrapolation}, we visualize the results up to 32,768 where we do see some performance degradation when the model needs to extrapolate. In addition, we observe that, despite often considered as having better extrapolation properties, \textsc{xPos ABF} does not outperform \textsc{RoPE ABF} in our setting. 
\end{itemize}

\section{Self-Instruct Data}
\label{sec:self_inst_appendix}

As described in Section \ref{sec:sft}, we use \chatllama to bootstrap self-instruct data for instruct finetuning. In this section we describe the detailed procedure as well as providing the necessary prompts used for generating this dataset. The main challenge is that we need an automated process of generating long context instruct data with only short context models at hand. The core idea behind this is to split the long documents into chunks of texts that can fit into short model's context and apply self-instruct. We focus primarily on question answering dataset. We first split the long document into smaller chunks, and for each chunk we construct a prompt as in Figure \ref{fig:qa_gen_prompts} which gets fed into \chatllama to get a question-answer pair. To diversify the question types, we randomly choose between the two prompts that ask for either normal or short answers. Once we extract the question and answer from the response (using tags as required by the prompt), we can construct long question answering instruct data together with the original long document, using the templates in Figure \ref{fig:qa_data_templates} of the corresponding answer type. 



\begin{figure}[ht]
\small
\noindent\rule{\textwidth}{1pt}
\textbf{Normal Answer Prompt:} 
\begin{verbatim}
[INST] You are given a text chunk (delimited by triple quotes) taken from a long 
text. Write a question about this text and provide the correct answer. The answer 
needs to be based on the text. This question will later be used as a reading 
comprehension test over the entire document. Wrap the question and answer using 
XML tags (<question> and </question>, <answer> and </answer>).
"""
{TEXT_CHUNK}
"""
[/INST]
\end{verbatim}
\noindent\rule{\textwidth}{0.5pt}
\textbf{Short Answer Prompt:} 
\begin{verbatim}
[INST] You are given a text chunk (delimited by triple quotes) from a long 
document. Based on information from the text, come up with a specific question 
**which can be answered in a few words or a single phrase** and provide the 
correct answer without explanation. The answer needs to be based on the text. 
This question will later be used as a reading comprehension test over the 
entire document. Wrap the question and answer using XML tags (<question> 
and </question>, <answer> and </answer>). Again, the answer needs to be short.
"""
{TEXT_CHUNK}
"""
[/INST]
\end{verbatim}
\noindent\rule{\textwidth}{1pt}
\caption{Prompts used for generating question and answer pairs by boostrapping \chatllama. We split the long documents into chunks and feed each chunk into one of the prompts with equal probability. We prompt the models to wrap the answer with XML tags, which enables more accurate answer extraction. }
\label{fig:qa_gen_prompts}
\end{figure}

\begin{figure}[ht]
\small
\noindent\rule{\textwidth}{1pt}
\textbf{Normal Answer Data Template:} 
\begin{verbatim}
[INST] You are given a long text (delimited by triple quotes) and a question. 
Read the text and answer the question in the end.
"""
{FULL_DOCUMENT}
"""
Question: {QUESTION} 
[/INST]
{ANSWER}
\end{verbatim}
\noindent\rule{\textwidth}{0.5pt}
\textbf{Short Answer Data Template:} 
\begin{verbatim}
[INST] You are given a long text (delimited by triple quotes) and a question. 
Read the text and answer the question in the end as concisely as you can, 
using a single phrase or sentence if possible. Do not provide any explanation.
"""
{FULL_DOCUMENT}
"""
Question: {QUESTION} 
[/INST]
{ANSWER}
\end{verbatim}
\noindent\rule{\textwidth}{1pt}
\caption{Data templates for constructing long question-answer data. The question and answer pair is extracted from the response of \chatllama.}
\label{fig:qa_data_templates}
\end{figure}

%% file: tables/compare_with_oss_32k.tex
\begin{table}[!ht]
  \small
  \centering
  \begin{adjustbox}{max width=1.2\linewidth, center}
      \begin{tabular}{lccccc}
        \toprule
        \multirow{2}{*}{\textbf{Model}}  &  \multirow{2}{*}{Prompt length} & NarrativeQA  & Qasper  & QuALITY & QMSum \\
        & &  F1 (0-shot) & F1 (2-shot) & EM (2-shot) & ROUGE-geo$^*$ (1-shot) \\
        \midrule

        Yarn-7B-128k & 16k & 20.9 & 26.2 & 32.3 & 11.4 \\
        Together-7B-32k & 16k  & 23.3 & 27.3 & 41.2 & 12.6\\
        Yarn-13B-128k & 16k & 23.4 & 27.1 & 46.4 & 11.9 \\
        Yarn-7B-128k  & 32k & 24.0 & 26.2 & 30.4 & 13.6 \\
        Together-7B-32k & 32k & 24.7 & 27.3 & 41.3 & 14.2\\
        Yarn-13B-128k & 32k & 25.5 & 27.1 & 48.0 & 13.8 \\
        \midrule

        \model 7B & 16k   & 21.9 & 27.8 & 43.2 & 14.9 \\
        \model 13B & 16k & 25.6 & 31.2 & 57.6 & 15.7 \\ 
        \model 7B & 32k & 24.4 & 28.7 & 43.6 & 15.9 \\ 
        \model 13B & 32k & \textbf{27.4} & \textbf{31.6} & \textbf{59.0} & \textbf{17.0} \\
        \bottomrule
      \end{tabular}
  \end{adjustbox}
  \vspace{5pt}
  \caption{Comparison of our models with open-source long-context models on research benchmarks \textit{using a maximum prompt length of 32,768 tokens}.
  }
  \label{table:compare_with_oss_32k}
\end{table}

%% file: tables/mmlu_decomposed_results.tex
\begin{table}[h]
  
  \centering
  \begin{tabular}{lcccc}
    \toprule
    \textbf{Model} & Humanities & STEM  & Social Sciences & Other \\
    \midrule
    \model 7B  & 54.8 & 35.7 & 58.4 & 53.2\\
    \model 13B & 69.0 & 44.4 & 71.3 & 65.8\\
    \model 34B & 73.5 & 49.9 & 78.4 & 69.3 \\
    \model 70B & 80.1 & 55.5 & 84.4 & 74.9 \\
    \bottomrule
  \end{tabular}
  \vspace{5pt}
  \caption{Decomposed MMLU results.}
  \label{table:mmlu_decomposed_results}
\end{table}

%% file: tables/more_decomposed_results.tex
\begin{table}[h]
  \centering
  \begin{tabular}{l|cc|cc|cc}
    \toprule
    \textbf{Model} & HumanEval & MBPP & MATH & GSM8k & NQ & TQA \\
    \midrule
    \model 7B  & 18.3 & 23.0 & 4.22 & 16.8 & 27.5 & 74.4\\
    \model 13B & 19.5 & 31.8 & 8.38 & 34.6 & 32.5 & 81.1\\
    \model 34B & 22.6 & 37.2 & 10.6 & 47.4 & 35.0 & 85.6\\
    \model 70B & 32.9 & 46.8 & 17.2 & 65.4 & 39.8 & 88.2\\
    \bottomrule
  \end{tabular}
  \vspace{5pt}
  \caption{Results on HumanEval (0-shot), MBPP (3-shot), MATH (4-shot), GSM8K (8-shot), NaturalQuestions (5-shot) and TriviaQA-wiki (5-shot).}
  \label{table:mmlu_decomposed_results}
\end{table}

%% file: tables/commonsense_decomposed.tex
\begin{table}[h!]
  \small
  \centering
  \begin{tabular}{lcccccccc}
    \toprule
    \textbf{Model} & PIQA & SIQA & HellaSwag & WinoGrande & ARC-e & ARC-c & OBQA & CSQA \\
    \midrule
    \model 7B  & 78.9 & 48.7 & 77.8 & 70.4 & 76.2 & 52.0 & 59.0 & 61.0 \\
    \model 13B & 81.6 & 50.7 & 81.2 & 74.1 & 77.7 & 51.4 & 55.6 & 70.4 \\
    \model 34B & 82.6 & 51.7 & 83.8 & 77.5 & 79.7 & 54.8 & 60.2 & 77.0 \\
    \model 70B & 83.3 & 52.8 & 85.7 & 79.6 & 80.3 & 58.4 & 59.6 & 81.9 \\
    \bottomrule
  \end{tabular}
  \vspace{5pt}
  \caption{Commonsense reasoning decomposed results. We use the same number of shots and evaluation metrics for all tasks as \llamavtwo. }
  \label{table:commonsense_decomposed_results}
\end{table}

%% file: tables/longeval_tasks.tex


\begin{table*}[ht]
  
  \small
  \centering
  \begin{adjustbox}{max width=1.2\linewidth, center}
  \begin{tabular}{l|ccccccc}
    \toprule
    \textbf{Model} & Coursera & TPO & TopicRetrieval & FinQA & ContractQA & NaturalQuestions \\
    \midrule
    Claude 1.3 100k      & 60.2 & 83.6 & 70.6 & - & - & -\\
    gpt-3.5-turbo-16k    & 59.7 & 69.9 & 69.3 & 45.4 & 24.9 & 45.9 \\
    \midrule
    \multicolumn{2}{l}{\textit{Best open models reported in \citet{leval}}}      \\
    longchat-13b-16k & 36.8 & 55.4 & 33.3 & 37.9 & 21.1 & 22.8 \\
    chatglm2-6b-8k & 47.2 & 54.6 & 10.0 & 34.8 & 16.4 & 17.6 \\
    \midrule
    \modelchat     & 52.9 & \underline{81.8} & \underline{76.0} & \underline{47.3} & \underline{25.5} & \underline{66.7} \\
    \bottomrule

  \end{tabular}
  \end{adjustbox}
      \caption{Evaluation on additional long-context tasks from L-Eval. We report the official metrics defined in \citet{leval} and the results of compared models are directly token from the paper.}
    \label{table:long_eval_results}
\end{table*}

%% file: sections/pe_theory_appendix.tex
RoPE maps an argument vector $x\in \mathbb R^d$ into the embedding curve on a sphere in $\mathbb C^{d/2}$ parametrized by a real parameter $t\in \mathbb R$ and ``base frequency'' $b$:

\begin{equation*}
f^{RoPE}(x, t)_j = \left(x_{2j} + i x_{2j+1}\right)e^{ib^{-\frac{2j}{d}}t}.    
\end{equation*}

The purpose of this mapping is to help the attention module to separate the vectors corresponding to two instances of the same token that are situated at different positions in the input sequence. 

Aiming at extending the sequence length of a transformer pretrained with a particular positional embedding $f$ from $L$ to $\hat L,$ we would like to come up with a positional embedding $\hat f$ that minimizes the distance between the old and the new images of the embedded vectors:
\[d(f, \hat f) = \max_{x\in \mathcal X}\min_{k\in \{0,..N-1\}~j\in \{0,..\hat N -1\}}\text{dist}[f(x, k), \hat f (x, j)],\]
where $\mathcal X \subset \mathbb R^d$ is the set of vectors that would need to be positionally embedded. \citep{pi} computed this distance through the magnitude of the attention scores, but still argued for the efficiency of their method ``position interpolation'') due to its reduced value of the distance to the original RoPE images when compared to the naive extrapolation of the positional embedding.

With this in mind, we consider two different methods to extend the sequence length of a trained transformer: Position Interpolation (PI) parameterized with $\alpha$,  and Adjusted Base Frequency (ABF) parameterized with $\beta.$ These two methods correspond to the following embedding curves:

\begin{align*} 
f^{RoPE+PI}(x,t)_j &= \left(x_{2j} +i x_{2j+1}\right)e^{i\alpha \cdot (b^{-\frac{2j}{d}})t } \\
f^{RoPE+ABF}(x,t)_j &= \left(x_{2j} + i x_{2j+1}\right)e^{i(\beta b)^{-\frac{2j}{d}}t}
\end{align*}

Evaluating a positional embedding a-priori, we should consider the degree of granularity with which the embedding images are being distributed over the embedding space. Comparing alternative positional embeddings $\hat f$ mapping $\mathbb R^d \times \mathbb N$ into $\mathbb C^{d/2},$  we should prefer the one with the maximal value of the distance between the two closest images:
\[q(\hat f) = \min_{x\in \mathcal X; k\ne j \in \{0..\hat N-1\} }\text{dist}[\hat f(x, k), \hat f(x, j)].\]
This leaves us with a multi-objective decision selecting the positional embedding for a model with extended context: on one hand, $\hat f$ should be chosen so that it minimizes $d(f, \hat f),$ while on the other hand its value of $q(\hat f)$ should be big enough.


Before proceeding to the explanation on how we make this multi-objective decision, we would like to provide a geometric intuition for the positional embeddings considered here.
While it is difficult to visualize a mapping $\mathbb R^d \times \mathbb N \to \mathbb C^{d/2},$ we can consider $x \in \mathbb R^d $ to be fixed and visualize the projection $\mathbb R \to \mathbb R^3.$ To get the intuition behind PI and ABF, let us consider the helix that is formed by Re$\left[f^{RoPE}(x, t)_0\right]$, Im$\left[f^{RoPE}(x, t)_0\right]$ and Re$\left[f^{RoPE}(x, t)_j\right].$ The example on the Figure \ref{fig:RoPE} depicts a black helix line given with the system
\[x = \textrm{cos } t;
y = \textrm{sin } t;
z = \textrm{sin } at.\]
The red dots on the line correspond to $11$ integer values of $t.$ 

Figure \ref{fig:RoPE_PI} aims to illustrate the impact of Position Interpolation on the relative position of the mapped vectors. The distance between the consecutive points got reduced considerably compered to Figure \ref{fig:RoPE}. The impact of Adjusted Base Frequency is illustrated on Figure \ref{fig:RoPE_BA}. The distance between the consecutive points remained almost the same as on Figure \ref{fig:RoPE}, although the minimal distance between points got considerably reduced due to the increased frequency of the helix. This effect of increased frequency of the helix would be reduced in the high dimension setting. The value of the coefficient $a$ for the helix depicted on Figure \ref{fig:RoPE} is two times larger than the value of the coefficient $a$ for the helix depicted on Figure \ref{fig:RoPE_BA}. If the dimension of the input of the attention mechanism is $d=128,$ then the difference between $\theta_1 = b^{-\frac{2j}{d}}$ at $b=10,000$ and $\theta_1 = b^{-\frac{2j}{d}}$ at $b=500,000$ is only $6\%.$ Thus, we further focus specifically on the distance between the consecutive images of the embeddings.

\begin{figure}
    \centering
    \begin{subfigure}[b]{0.3\textwidth}
         \centering
         \includegraphics[width=1\linewidth]{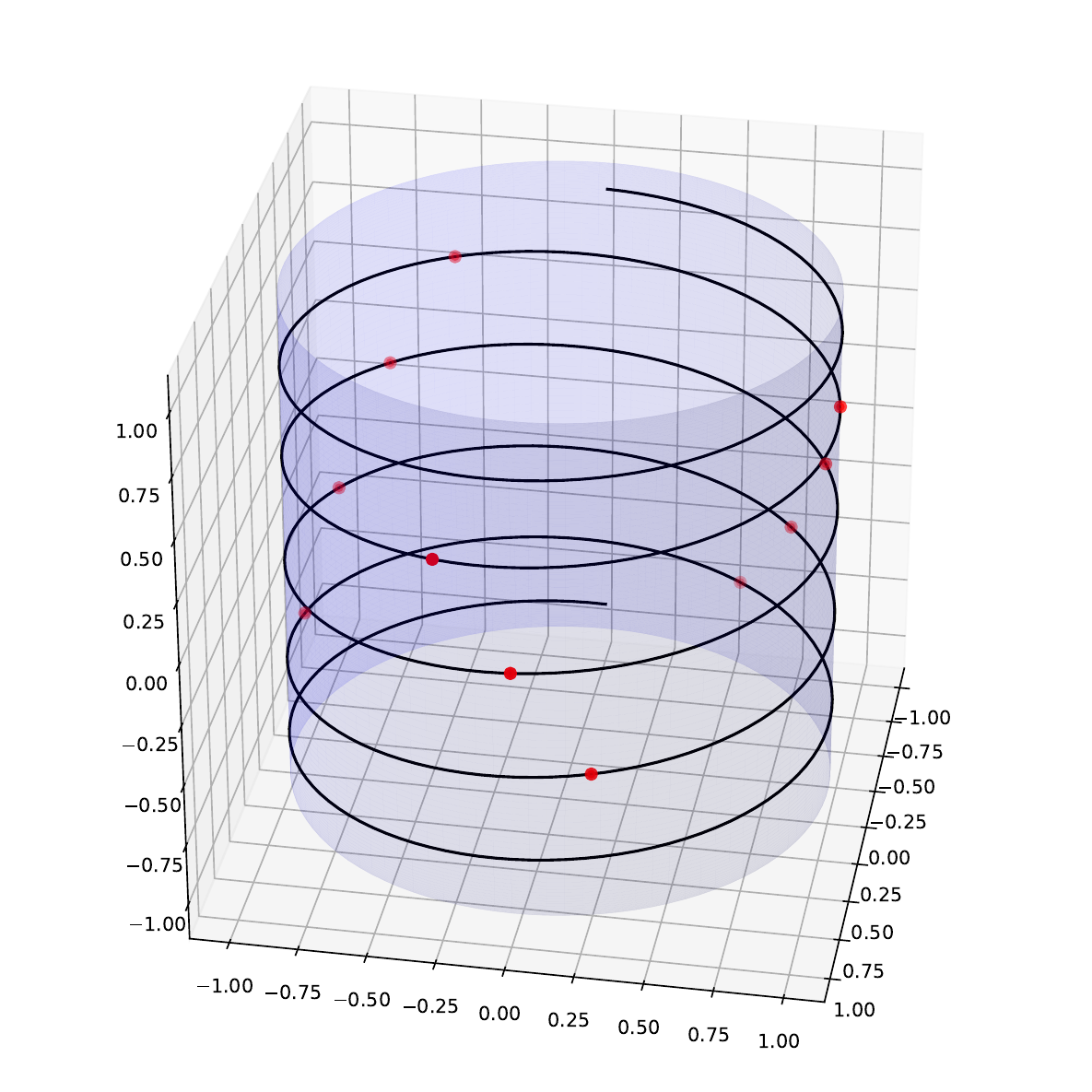}
         \caption{RoPE}
         \label{fig:RoPE}
     \end{subfigure}
     \begin{subfigure}[b]{0.3\textwidth}
         \centering
         \includegraphics[width=1\linewidth]{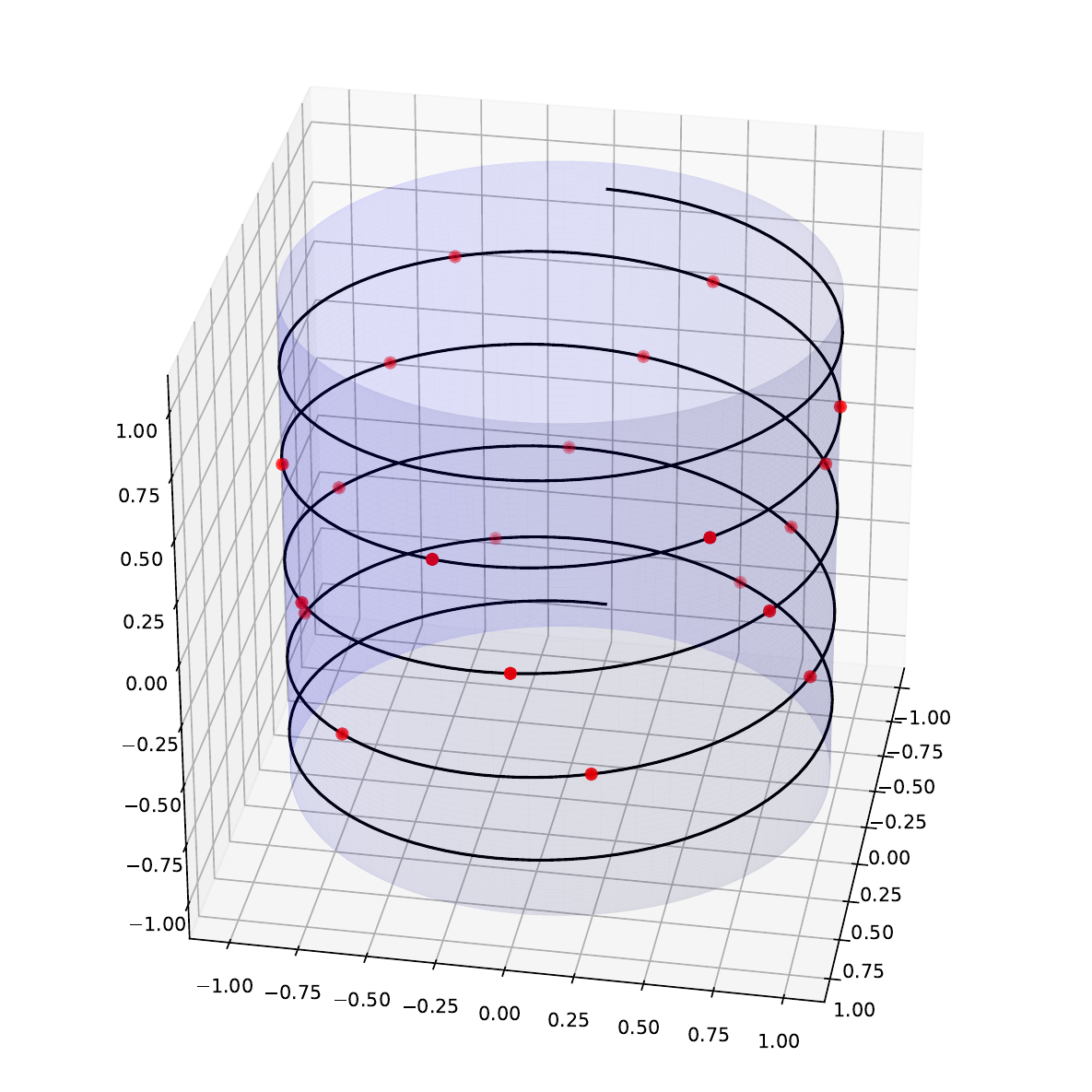}
        \caption{RoPE+PI}
    \label{fig:RoPE_PI}
     \end{subfigure}
     \begin{subfigure}[b]{0.3\textwidth}
         \centering
         \includegraphics[width=1\linewidth]{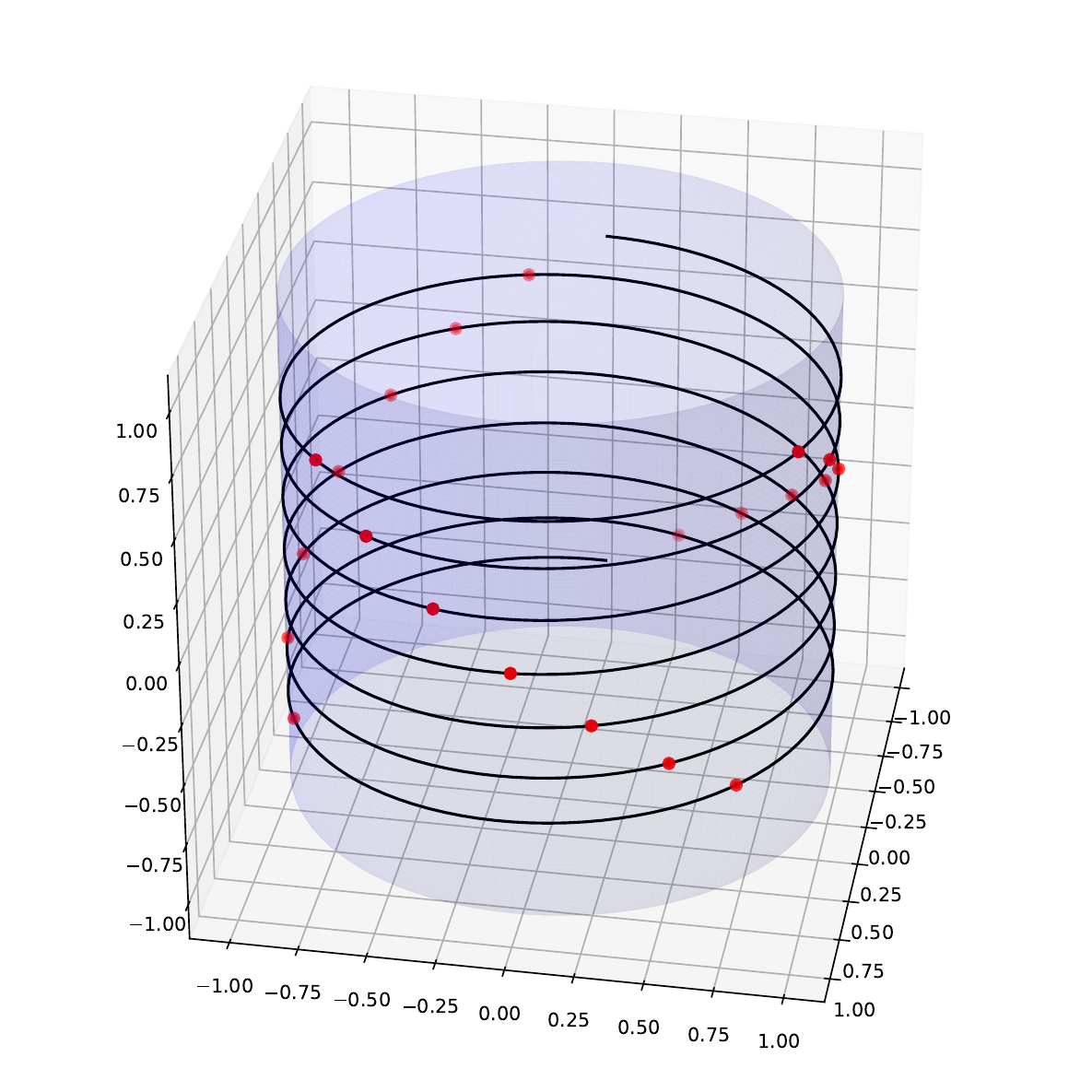}
        \caption{RoPE+ABF}
    \label{fig:RoPE_BA}
     \end{subfigure}
\caption{RoPE variants visualization as helices.}
\end{figure}

We make a formal comparison between Positional Interpolation and Adjusted Base Frequency by analytically comparing the pairwise distances between the images given by $f^{RoPE+PI}$ and $f^{RoPE+ABF}$ for consecutive integer values of $t$. This corresponds to the evaluation of $q(\hat f)$ discussed earlier.
We will measure the distance between embedding images in terms of the Euclidean sine similarity metric since all versions of RoPE are norm-preserving.
\[\sin \angle (a, b) = \frac{Im \langle a, b\rangle}{\|a\|\|b\|}\]

The following result states that in a high-dimensional space, the sine similarity $\sin \angle (f^{RoPE + ABF}(x, n+1), f^{RoPE + ABF}(x, n))$ between two consecutive embedding images of a vector $x$ can be bounded with a value proportional to $(\log b + \log \beta)^{-1}.$ Moreover, the similarity $\sin \angle (f^{RoPE + PI}(x, n+1), f^{RoPE + PI}(x, n))$ can be bounded using $\alpha(\log b)^{-1}.$ 

\begin{theorem}
    For $x \in \mathbb R^d $ and $n \in \mathbb N,$ the Euclidean sine similarity between the two consecutive images of a positional embedding can be bounded as

    $\frac{\min_k x_k^2}{\|x\|^2} C_d \le \sin \angle (f(x, n+1), f(x, n)) \le \frac{\max_k x_k^2}{\|x\|^2} C_d$

    where $\lim_{d\to \infty}C_d \approx \begin{cases}
        (\log b + \log \beta)^{-1} \text{ if } f = f^{RoPE + ABF} \\
        \alpha (\log b)^{-1} \text{ if } f = f^{RoPE + PI}
    \end{cases}$  under the assumptions of $\alpha \ll 1$ and $b\gg 1.$
\end{theorem}

\begin{proof}

Let us begin the proof by writing down the expressions for the inner product between two images of RoPE variants.


$\langle f^{RoPE+PI}(x, m), f^{RoPE+PI}(x, n)\rangle = \sum_{j=0}^{\frac{d}{2}-1}\left(x_{2j}^2 + x_{2j+1}^2\right)e^{ib^{-\frac{2j}{d}}\alpha(m-n)}$

$\langle f^{RoPE+ABF}(x, m), f^{RoPE+ABF}(x, n)\rangle = \sum_{j=0}^{\frac{d}{2}-1}\left(x_{2j}^2 + x_{2j+1}^2\right)e^{ib^{-\frac{2j}{d}}\beta^{-\frac{2j}{d}}(m-n)}$

From them, we can derive the expressions for the Euclidean sine similarity between the images of the positional embeddings:


$\sin \angle (f^{RoPE+PI}(x, m), f^{RoPE+PI}(x, n)) =  \frac{\sum_{j=0}^{\frac{d}{2}-1}\left(x_{2j}^2 + x_{2j+1}^2\right)\sin(b^{-\frac{2j}{d}}\alpha(m-n))}{\sum_{j=0}^{d-1} x_j^2} $

$\sin \angle (f^{RoPE+ABF}(x, m), f^{RoPE+ABF}(x, n)) = \frac{\sum_{j=0}^{\frac{d}{2}-1}\left(x_{2j}^2 + x_{2j+1}^2\right)\sin(b^{-\frac{2j}{d}}\beta^{-\frac{2j}{d}}(m-n))}{\sum_{j=0}^{d-1} x_j^2} $

Let's put $m=n+1$ to compare the distance between the two consecutive positional embedding images of the same vector $x.$

$\|x\|^2 \sin \angle (f^{RoPE+PI}(x, n+1), f^{RoPE+PI}(x, n)) = \sum_{j=0}^{\frac{d}{2}-1}\left(x_{2j}^2 + x_{2j+1}^2\right)\sin(b^{-\frac{2j}{d}}\alpha)$

$\|x\|^2 \sin \angle (f^{RoPE+ABF}(x, n+1), f^{RoPE+ABF}(x, n)) = \sum_{j=0}^{\frac{d}{2}-1}\left(x_{2j}^2 + x_{2j+1}^2\right)\sin(b^{-\frac{2j}{d}}\beta^{-\frac{2j}{d}})$
Due to the range of $b, \alpha$ and $\beta$ that is typically considered, we can bound the arguments of the sine functions as  $0<\alpha b^{-\frac{2j}{d}}\le 1$ as well as $0<(\beta b)^{-\frac{2j}{d}} \le 1.$ 
Using that, we derive that $\sin (b^{-\frac{2j}{d}}\beta^{-\frac{2j}{d}})$ and $\sin(b^{-\frac{2j}{d}}\alpha)$ are non-negative as well as $x_j^2$ for any $j\in \{1,\ldots d\}.$  Thus, the following inequalities hold:

\[\sum_{j=0}^{\frac{d}{2}-1}\min_k x_k^2\sin(b^{-\frac{2j}{d}}\beta^{-\frac{2j}{d}}) \le \sum_{j=0}^{\frac{d}{2}-1}\left(x_{2j}^2 + x_{2j+1}^2\right)\sin(b^{-\frac{2j}{d}}\beta^{-\frac{2j}{d}}) \le \sum_{j=0}^{\frac{d}{2}-1}\max_k x_k^2\sin(b^{-\frac{2j}{d}}\beta^{-\frac{2j}{d}}),\]

\[ \sum_{j=0}^{\frac{d}{2}-1}\min_k x_k^2\sin(b^{-\frac{2j}{d}}\alpha) \le \sum_{j=0}^{\frac{d}{2}-1}\left(x_{2j}^2 + x_{2j+1}^2\right)\sin(b^{-\frac{2j}{d}}\alpha) \le \sum_{j=0}^{\frac{d}{2}-1}\max_k x_k^2\sin(b^{-\frac{2j}{d}}\alpha).\]
Carrying $\min_k x_k^2$ and $\max_k x_k^2$ out of the summation signs, we obtain

\[\min_k x_k^2 \sum_{j=0}^{\frac{d}{2}-1}\sin(b^{-\frac{2j}{d}}\beta^{-\frac{2j}{d}}) \le \sum_{j=0}^{\frac{d}{2}-1}\left(x_{2j}^2 + x_{2j+1}^2\right)\sin(b^{-\frac{2j}{d}}\beta^{-\frac{2j}{d}}) \le \max_k x_k^2 \sum_{j=0}^{\frac{d}{2}-1}\sin(b^{-\frac{2j}{d}}\beta^{-\frac{2j}{d}}),\]

\[\min_k x_k^2 \sum_{j=0}^{\frac{d}{2}-1}\sin(b^{-\frac{2j}{d}}\alpha) \le \sum_{j=0}^{\frac{d}{2}-1}\left(x_{2j}^2 + x_{2j+1}^2\right)\sin(b^{-\frac{2j}{d}}\alpha) \le \max_k x_k^2 \sum_{j=0}^{\frac{d}{2}-1}\sin(b^{-\frac{2j}{d}}\alpha).\]

Introducing $C^{ABF}_d = \sum_{j=0}^{\frac{d}{2}-1}\sin(b^{-\frac{2j}{d}}\beta^{-\frac{2j}{d}})$ and $C^{PI}_d =\sum_{j=0}^{\frac{d}{2}-1}\sin(b^{-\frac{2j}{d}}\alpha)$ proves the first part of the Theorem: 
\[\frac{\min_k x_k^2}{\|x\|^2} C^{ABF}_d \le \sin \angle (f^{RoPE+ABF}(x, n+1), f^{RoPE+ABF}(x, n)) \le \frac{\max_k x_k^2}{\|x\|^2} C^{ABF}_d,\]

\[\frac{\min_k x_k^2}{\|x\|^2} C^{PI}_d \le \sin \angle (f^{RoPE+PI}(x, n+1), f^{RoPE+PI}(x, n)) \le \frac{\max_k x_k^2}{\|x\|^2}C^{PI}_d.\]

Now, considering the limit of $C_d,$ we notice that  due to the inequalities on the arguments of the sines, the following bounds hold:

\[  (b\beta)^{-\frac{2j}{d}}\left(1-(b\beta)^{-\frac{2j}{d}}/\pi\right)\le\sin(b^{-\frac{2j}{d}}\beta^{-\frac{2j}{d}}) \le (b\beta)^{-\frac{2j}{d}},\]
\[  \alpha b^{-\frac{2j}{d}}\left(1-\alpha b^{-\frac{2j}{d}}/\pi\right)\le\sin(b^{-\frac{2j}{d}}\alpha) \le \alpha b^{-\frac{2j}{d}}\]

Using the formula of geometric sums and a corollary of the exponential (second) foundational limit, we establish the limits of the sums of these bounds as $d\to \infty$:
\[\sum_{j=0}^{\frac{d}{2}-1}\alpha b^{-\frac{2j}{d}} = \frac{\alpha(b-1)b^{2/d}}{b^{2/d+1}-b}\to \alpha \frac{b-1}{b\log b} \text{ as } d\to \infty\]
\[\sum_{j=0}^{\frac{d}{2}-1}\alpha^2 b^{-\frac{4j}{d}} = \frac{\alpha^2(b^2-1)b^{4/d}}{b^{4/d+2}-b^2} \to \alpha^2 \frac{b^2-1}{b^2\log b} \text{ as } d\to \infty \]
\[\sum_{j=0}^{\frac{d}{2}-1} (b\beta)^{-\frac{2j}{d}} = \frac{(b\beta-1)(b\beta)^{2/d}}{(b\beta)^{2/d+1}-b\beta} \to \frac{(b\beta)-1}{(b\beta)\log (b\beta)} \text{ as } d\to \infty  \]

\[\sum_{j=0}^{\frac{d}{2}-1} (b\beta)^{-\frac{4j}{d}} = \frac{(b^2\beta^2-1)(b\beta)^{4/d}}{(b\beta)^{4/d+2}-b^2\beta^2} \to \frac{(b\beta)^2-1}{(b\beta)^2\log (b\beta)} \text{ as } d\to \infty \]
Substituting these into the bounds on $\lim_{d\to\infty} C_d,$ one achieves:
\[  (\log b +\log \beta)^{-1} \left(\frac{(b\beta)-1}{(b\beta)} - \frac{(b\beta)^2-1}{\pi(b\beta)^2} \right ) \le \lim_{d\to\infty} C^{ABF}_d \le (\log b +\log \beta)^{-1}\frac{(b\beta)-1}{(b\beta)} ,\]
\[  \alpha (\log b)^{-1}\left(\frac{b-1}{b} - \frac{\alpha}{\pi} \frac{b^2-1}{b^2} \right) \le \lim_{d\to\infty} C^{PI}_d \le \alpha (\log b)^{-1} \frac{b-1}{b}\]
From these bounds, one can see that in the setting considered within this paper, where $b=10000$ and $\alpha < 1/4,$ the approximation of $\lim_{d\to \infty} C_d$ used in the statement of the Theorem is of a high quality.

\end{proof}

Based on this theoretical derivation, we return to the interpretation of our experimental resuts.
On one hand, the experiments have shown that the model can adapt to the new sequence length with both RoPE PI ($\alpha = 1/4$ or $\alpha=1/8$) and RoPE ABF ($\beta=50$). Thus, we can conclude that the chosen hyperparameters provide a sufficient degree of approximation of RoPE images under $b=10000.$ In other words, both $d(f,  f^{RoPE+ABF})$ and $d(f,  f^{RoPE+PI})$ are small enough to allow rapid adaptation. On the other hand, comparing the expressions of $C_d$ for RoPE ABF and RoPE PI, we can observe that for the values of $\alpha = \frac{1}{4}$ or $\alpha = \frac{1}{8}$  and $b=10000$ that were used in our experiments, the granularity (the distance between two consecutive images of RoPE) is much lower for the RoPE PI ($\alpha (\log b)^{-1} \approx 0.027$) than for RoPE ABF ($(\log b + \log \beta )^{-1} \approx 0.076$) with $\beta=50.$ We further hypothesise that the higher degree of granularity is related to the higher evaluation on the downstream tasks of the RoPE ABF variant compared to RoPE PI because it makes the task of distinguishing between the positional embedding images simpler for the model. In other words, this corresponds to the case of $q(f^{RoPE+ABF}) > q(f^{RoPE+PI}).$

Throughout this consideration we implicitly assumed that the distance between the consecutive images of an embedding is smaller than the distance between any other pair of the images. While this assumption is likely to hold true in a high-dimensional space, significantly increasing the parameter of $\beta$ in RoPE ABF may violate this assumption due to the changed geometry of the embedding curve.

%% file: main.bbl
\begin{thebibliography}{53}
\providecommand{\natexlab}[1]{#1}
\providecommand{\url}[1]{\texttt{#1}}
\expandafter\ifx\csname urlstyle\endcsname\relax
  \providecommand{\doi}[1]{doi: #1}\else
  \providecommand{\doi}{doi: \begingroup \urlstyle{rm}\Url}\fi

\bibitem[Almazrouei et~al.(2023)Almazrouei, Alobeidli, Alshamsi, Cappelli,
  Cojocaru, Debbah, Goffinet, Heslow, Launay, Malartic, Noune, Pannier, and
  Penedo]{falcon40b}
Ebtesam Almazrouei, Hamza Alobeidli, Abdulaziz Alshamsi, Alessandro Cappelli,
  Ruxandra Cojocaru, Merouane Debbah, Etienne Goffinet, Daniel Heslow, Julien
  Launay, Quentin Malartic, Badreddine Noune, Baptiste Pannier, and Guilherme
  Penedo.
\newblock {Falcon-40B}: an open large language model with state-of-the-art
  performance.
\newblock 2023.

\bibitem[An et~al.(2023)An, Gong, Zhong, Li, Zhang, Kong, and Qiu]{leval}
Chenxin An, Shansan Gong, Ming Zhong, Mukai Li, Jun Zhang, Lingpeng Kong, and
  Xipeng Qiu.
\newblock L-eval: Instituting standardized evaluation for long context language
  models.
\newblock \emph{arXiv preprint arXiv:2307.11088}, 2023.

\bibitem[{Anthropic}(2023)]{claude}
{Anthropic}.
\newblock Introducing {{100K Context Windows}}, 2023.
\newblock URL \url{https://www.anthropic.com/index/100k-context-windows}.

\bibitem[Austin et~al.(2021)Austin, Odena, Nye, Bosma, Michalewski, Dohan,
  Jiang, Cai, Terry, Le, and Sutton]{mbpp}
Jacob Austin, Augustus Odena, Maxwell~I. Nye, Maarten Bosma, Henryk
  Michalewski, David Dohan, Ellen Jiang, Carrie~J. Cai, Michael Terry, Quoc~V.
  Le, and Charles Sutton.
\newblock Program synthesis with large language models.
\newblock \emph{arXiv:abs/2108.07732}, 2021.

\bibitem[Bisk et~al.(2020)Bisk, Zellers, Gao, Choi, et~al.]{piqa}
Yonatan Bisk, Rowan Zellers, Jianfeng Gao, Yejin Choi, et~al.
\newblock Piqa: Reasoning about physical commonsense in natural language.
\newblock In \emph{Proceedings of the AAAI conference on artificial
  intelligence}, volume~34, pages 7432--7439, 2020.

\bibitem[Chen et~al.(2021)Chen, Tworek, Jun, Yuan, Pinto, Kaplan, Edwards,
  Burda, Joseph, Brockman, et~al.]{humaneval}
Mark Chen, Jerry Tworek, Heewoo Jun, Qiming Yuan, Henrique Ponde de~Oliveira
  Pinto, Jared Kaplan, Harri Edwards, Yuri Burda, Nicholas Joseph, Greg
  Brockman, et~al.
\newblock Evaluating large language models trained on code.
\newblock \emph{arXiv preprint arXiv:2107.03374}, 2021.

\bibitem[Chen et~al.(2023)Chen, Wong, Chen, and Tian]{pi}
Shouyuan Chen, Sherman Wong, Liangjian Chen, and Yuandong Tian.
\newblock Extending context window of large language models via positional
  interpolation, 2023.

\bibitem[Child et~al.(2019)Child, Gray, Radford, and Sutskever]{blocksparse}
Rewon Child, Scott Gray, Alec Radford, and Ilya Sutskever.
\newblock Generating long sequences with sparse transformers.
\newblock \emph{arXiv preprint arXiv:1904.10509}, 2019.

\bibitem[Clark et~al.(2018)Clark, Cowhey, Etzioni, Khot, Sabharwal, Schoenick,
  and Tafjord]{arc}
Peter Clark, Isaac Cowhey, Oren Etzioni, Tushar Khot, Ashish Sabharwal, Carissa
  Schoenick, and Oyvind Tafjord.
\newblock Think you have solved question answering? try arc, the ai2 reasoning
  challenge.
\newblock \emph{arXiv preprint arXiv:1803.05457}, 2018.

\bibitem[Cobbe et~al.(2021)Cobbe, Kosaraju, Bavarian, Chen, Jun, Kaiser,
  Plappert, Tworek, Hilton, Nakano, et~al.]{gsm8k}
Karl Cobbe, Vineet Kosaraju, Mohammad Bavarian, Mark Chen, Heewoo Jun, Lukasz
  Kaiser, Matthias Plappert, Jerry Tworek, Jacob Hilton, Reiichiro Nakano,
  et~al.
\newblock Training verifiers to solve math word problems.
\newblock \emph{arXiv preprint arXiv:2110.14168}, 2021.

\bibitem[Conover et~al.(2023)Conover, Hayes, Mathur, Xie, Wan, Shah, Ghodsi,
  Wendell, Zaharia, and Xin]{dolly_v2}
Mike Conover, Matt Hayes, Ankit Mathur, Jianwei Xie, Jun Wan, Sam Shah, Ali
  Ghodsi, Patrick Wendell, Matei Zaharia, and Reynold Xin.
\newblock Free dolly: Introducing the world's first truly open
  instruction-tuned llm, 2023.
\newblock URL
  \url{https://www.databricks.com/blog/2023/04/12/dolly-first-open-commercially-viable-instruction-tuned-llm}.

\bibitem[Dao et~al.(2022)Dao, Fu, Ermon, Rudra, and R{\'{e}}]{flashattention}
Tri Dao, Daniel~Y. Fu, Stefano Ermon, Atri Rudra, and Christopher R{\'{e}}.
\newblock Flashattention: Fast and memory-efficient exact attention with
  io-awareness.
\newblock In \emph{NeurIPS}, 2022.

\bibitem[Dasigi et~al.(2021)Dasigi, Lo, Beltagy, Cohan, Smith, and
  Gardner]{qasper}
Pradeep Dasigi, Kyle Lo, Iz~Beltagy, Arman Cohan, Noah~A. Smith, and Matt
  Gardner.
\newblock A dataset of information-seeking questions and answers anchored in
  research papers.
\newblock In \emph{Proceedings of the 2021 Conference of the North American
  Chapter of the Association for Computational Linguistics: Human Language
  Technologies}, pages 4599--4610, Online, June 2021. Association for
  Computational Linguistics.
\newblock \doi{10.18653/v1/2021.naacl-main.365}.
\newblock URL \url{https://aclanthology.org/2021.naacl-main.365}.

\bibitem[Dhamala et~al.(2021)Dhamala, Sun, Kumar, Krishna, Pruksachatkun,
  Chang, and Gupta]{dhamala2021bold}
Jwala Dhamala, Tony Sun, Varun Kumar, Satyapriya Krishna, Yada Pruksachatkun,
  Kai-Wei Chang, and Rahul Gupta.
\newblock Bold: Dataset and metrics for measuring biases in open-ended language
  generation.
\newblock In \emph{Proceedings of the 2021 ACM conference on fairness,
  accountability, and transparency}, pages 862--872, 2021.

\bibitem[Ding et~al.(2023)Ding, Ma, Dong, Zhang, Huang, Wang, and
  Wei]{long_net}
Jiayu Ding, Shuming Ma, Li~Dong, Xingxing Zhang, Shaohan Huang, Wenhui Wang,
  and Furu Wei.
\newblock Longnet: Scaling transformers to 1,000,000,000 tokens, 2023.

\bibitem[Greshake et~al.(2023)Greshake, Abdelnabi, Mishra, Endres, Holz, and
  Fritz]{greshake2023not}
Kai Greshake, Sahar Abdelnabi, Shailesh Mishra, Christoph Endres, Thorsten
  Holz, and Mario Fritz.
\newblock Not what you’ve signed up for: Compromising real-world
  llm-integrated applications with indirect prompt injection.
\newblock \emph{arXiv preprint arXiv:2302.12173}, 2023.

\bibitem[Hartvigsen et~al.(2022)Hartvigsen, Gabriel, Palangi, Sap, Ray, and
  Kamar]{hartvigsen2022toxigen}
Thomas Hartvigsen, Saadia Gabriel, Hamid Palangi, Maarten Sap, Dipankar Ray,
  and Ece Kamar.
\newblock Toxigen: A large-scale machine-generated dataset for adversarial and
  implicit hate speech detection.
\newblock \emph{arXiv preprint arXiv:2203.09509}, 2022.

\bibitem[Hendrycks et~al.(2021)Hendrycks, Burns, Kadavath, Arora, Basart, Tang,
  Song, and Steinhardt]{math}
Dan Hendrycks, Collin Burns, Saurav Kadavath, Akul Arora, Steven Basart, Eric
  Tang, Dawn Song, and Jacob Steinhardt.
\newblock Measuring mathematical problem solving with the math dataset.
\newblock \emph{arXiv preprint arXiv:2103.03874}, 2021.

\bibitem[Hoffmann et~al.(2022)Hoffmann, Borgeaud, Mensch, Buchatskaya, Cai,
  Rutherford, de~Las~Casas, Hendricks, Welbl, Clark, Hennigan, Noland,
  Millican, van~den Driessche, Damoc, Guy, Osindero, Simonyan, Elsen, Rae,
  Vinyals, and Sifre]{hoffmann2022training}
Jordan Hoffmann, Sebastian Borgeaud, Arthur Mensch, Elena Buchatskaya, Trevor
  Cai, Eliza Rutherford, Diego de~Las~Casas, Lisa~Anne Hendricks, Johannes
  Welbl, Aidan Clark, Tom Hennigan, Eric Noland, Katie Millican, George van~den
  Driessche, Bogdan Damoc, Aurelia Guy, Simon Osindero, Karen Simonyan, Erich
  Elsen, Jack~W. Rae, Oriol Vinyals, and Laurent Sifre.
\newblock Training compute-optimal large language models, 2022.

\bibitem[Hutto and Gilbert(2014)]{hutto2014vader}
Clayton Hutto and Eric Gilbert.
\newblock Vader: A parsimonious rule-based model for sentiment analysis of
  social media text.
\newblock In \emph{Proceedings of the international AAAI conference on web and
  social media}, volume~8, pages 216--225, 2014.

\bibitem[Ji et~al.(2023)Ji, Lee, Frieske, Yu, Su, Xu, Ishii, Bang, Madotto, and
  Fung]{ji2023survey}
Ziwei Ji, Nayeon Lee, Rita Frieske, Tiezheng Yu, Dan Su, Yan Xu, Etsuko Ishii,
  Ye~Jin Bang, Andrea Madotto, and Pascale Fung.
\newblock Survey of hallucination in natural language generation.
\newblock \emph{ACM Computing Surveys}, 55\penalty0 (12):\penalty0 1--38, 2023.

\bibitem[Joshi et~al.(2017)Joshi, Choi, Weld, and Zettlemoyer]{triviaqa}
Mandar Joshi, Eunsol Choi, Daniel~S Weld, and Luke Zettlemoyer.
\newblock Triviaqa: A large scale distantly supervised challenge dataset for
  reading comprehension.
\newblock \emph{arXiv preprint arXiv:1705.03551}, 2017.

\bibitem[Kaplan et~al.(2020)Kaplan, McCandlish, Henighan, Brown, Chess, Child,
  Gray, Radford, Wu, and Amodei]{kaplan2020scaling}
Jared Kaplan, Sam McCandlish, Tom Henighan, Tom~B. Brown, Benjamin Chess, Rewon
  Child, Scott Gray, Alec Radford, Jeffrey Wu, and Dario Amodei.
\newblock Scaling laws for neural language models, 2020.

\bibitem[Ko{\v{c}}isk{\'y} et~al.(2018)Ko{\v{c}}isk{\'y}, Schwarz, Blunsom,
  Dyer, Hermann, Melis, and Grefenstette]{narrativeqa}
Tom{\'a}{\v{s}} Ko{\v{c}}isk{\'y}, Jonathan Schwarz, Phil Blunsom, Chris Dyer,
  Karl~Moritz Hermann, G{\'a}bor Melis, and Edward Grefenstette.
\newblock The {N}arrative{QA} reading comprehension challenge.
\newblock \emph{Transactions of the Association for Computational Linguistics},
  6:\penalty0 317--328, 2018.
\newblock \doi{10.1162/tacl_a_00023}.
\newblock URL \url{https://aclanthology.org/Q18-1023}.

\bibitem[K{\"o}pf et~al.(2023)K{\"o}pf, Kilcher, von R{\"u}tte, Anagnostidis,
  Tam, Stevens, Barhoum, Duc, Stanley, Nagyfi, et~al.]{oasis}
Andreas K{\"o}pf, Yannic Kilcher, Dimitri von R{\"u}tte, Sotiris Anagnostidis,
  Zhi-Rui Tam, Keith Stevens, Abdullah Barhoum, Nguyen~Minh Duc, Oliver
  Stanley, Rich{\'a}rd Nagyfi, et~al.
\newblock Openassistant conversations--democratizing large language model
  alignment.
\newblock \emph{arXiv preprint arXiv:2304.07327}, 2023.

\bibitem[Kwiatkowski et~al.(2019)Kwiatkowski, Palomaki, Redfield, Collins,
  Parikh, Alberti, Epstein, Polosukhin, Devlin, Lee, et~al.]{natural_questions}
Tom Kwiatkowski, Jennimaria Palomaki, Olivia Redfield, Michael Collins, Ankur
  Parikh, Chris Alberti, Danielle Epstein, Illia Polosukhin, Jacob Devlin,
  Kenton Lee, et~al.
\newblock Natural questions: a benchmark for question answering research.
\newblock \emph{Transactions of the Association for Computational Linguistics},
  7:\penalty0 453--466, 2019.

\bibitem[Lin et~al.(2021)Lin, Hilton, and Evans]{lin2021truthfulqa}
Stephanie Lin, Jacob Hilton, and Owain Evans.
\newblock Truthfulqa: Measuring how models mimic human falsehoods.
\newblock \emph{arXiv preprint arXiv:2109.07958}, 2021.

\bibitem[Liu et~al.(2019)Liu, Ott, Goyal, Du, Joshi, Chen, Levy, Lewis,
  Zettlemoyer, and Stoyanov]{liu2019roberta}
Yinhan Liu, Myle Ott, Naman Goyal, Jingfei Du, Mandar Joshi, Danqi Chen, Omer
  Levy, Mike Lewis, Luke Zettlemoyer, and Veselin Stoyanov.
\newblock Roberta: A robustly optimized bert pretraining approach.
\newblock \emph{arXiv preprint arXiv:1907.11692}, 2019.

\bibitem[Mihaylov et~al.(2018)Mihaylov, Clark, Khot, and Sabharwal]{OpenbookQA}
Todor Mihaylov, Peter Clark, Tushar Khot, and Ashish Sabharwal.
\newblock Can a suit of armor conduct electricity? a new dataset for open book
  question answering.
\newblock \emph{arXiv preprint arXiv:1809.02789}, 2018.

\bibitem[Mohtashami and Jaggi(2023)]{passkey}
Amirkeivan Mohtashami and Martin Jaggi.
\newblock Landmark attention: Random-access infinite context length for
  transformers.
\newblock \emph{arXiv preprint arXiv:2305.16300}, 2023.

\bibitem[MosaicML(2023{\natexlab{a}})]{MPT30b}
MosaicML.
\newblock Introducing mpt-30b: Raising the bar for open-source foundation
  models, 2023{\natexlab{a}}.
\newblock URL \url{www.mosaicml.com/blog/mpt-30b}.
\newblock Accessed: 2023-06-22.

\bibitem[MosaicML(2023{\natexlab{b}})]{MPT7b}
MosaicML.
\newblock Introducing mpt-7b: A new standard for open-source, ly usable llms,
  2023{\natexlab{b}}.
\newblock URL \url{www.mosaicml.com/blog/mpt-7b}.

\bibitem[Narayanan et~al.(2021)Narayanan, Shoeybi, Casper, LeGresley, Patwary,
  Korthikanti, Vainbrand, Kashinkunti, Bernauer, Catanzaro, et~al.]{megatron}
Deepak Narayanan, Mohammad Shoeybi, Jared Casper, Patrick LeGresley, Mostofa
  Patwary, Vijay Korthikanti, Dmitri Vainbrand, Prethvi Kashinkunti, Julie
  Bernauer, Bryan Catanzaro, et~al.
\newblock Efficient large-scale language model training on gpu clusters using
  megatron-lm.
\newblock In \emph{Proceedings of the International Conference for High
  Performance Computing, Networking, Storage and Analysis}, pages 1--15, 2021.

\bibitem[Nijkamp et~al.(2023)Nijkamp, Xie, Hayashi, Pang, Xia, Xing, Vig,
  Yavuz, Laban, Krause, et~al.]{xgen}
Erik Nijkamp, Tian Xie, Hiroaki Hayashi, Bo~Pang, Congying Xia, Chen Xing,
  Jesse Vig, Semih Yavuz, Philippe Laban, Ben Krause, et~al.
\newblock Long sequence modeling with xgen: A 7b llm trained on 8k input
  sequence length.
\newblock \emph{Salesforce AI Research Blog}, 2023.

\bibitem[OpenAI(2023)]{gpt-4}
OpenAI.
\newblock Gpt-4 technical report, 2023.

\bibitem[Ouyang et~al.(2022)Ouyang, Wu, Jiang, Almeida, Wainwright, Mishkin,
  Zhang, Agarwal, Slama, Ray, et~al.]{rlhf}
Long Ouyang, Jeffrey Wu, Xu~Jiang, Diogo Almeida, Carroll Wainwright, Pamela
  Mishkin, Chong Zhang, Sandhini Agarwal, Katarina Slama, Alex Ray, et~al.
\newblock Training language models to follow instructions with human feedback.
\newblock \emph{Advances in Neural Information Processing Systems},
  35:\penalty0 27730--27744, 2022.

\bibitem[Pang et~al.(2022)Pang, Parrish, Joshi, Nangia, Phang, Chen,
  Padmakumar, Ma, Thompson, He, and Bowman]{quality}
Richard~Yuanzhe Pang, Alicia Parrish, Nitish Joshi, Nikita Nangia, Jason Phang,
  Angelica Chen, Vishakh Padmakumar, Johnny Ma, Jana Thompson, He~He, and
  Samuel Bowman.
\newblock {Q}u{ALITY}: Question answering with long input texts, yes!
\newblock In \emph{Proceedings of the 2022 Conference of the North American
  Chapter of the Association for Computational Linguistics: Human Language
  Technologies}, pages 5336--5358, Seattle, United States, July 2022.
  Association for Computational Linguistics.
\newblock \doi{10.18653/v1/2022.naacl-main.391}.
\newblock URL \url{https://aclanthology.org/2022.naacl-main.391}.

\bibitem[Peng et~al.(2023)Peng, Quesnelle, Fan, and Shippole]{yarn}
Bowen Peng, Jeffrey Quesnelle, Honglu Fan, and Enrico Shippole.
\newblock Yarn: Efficient context window extension of large language models,
  2023.

\bibitem[r/LocalLLaMa()]{ntk_link}
r/LocalLLaMa.
\newblock {NTK-Aware Scaled RoPE} allows llama models to have extended (8k+)
  context size without any fine-tuning and minimal perplexity degradation.
\newblock
  \url{https://www.reddit.com/r/LocalLLaMA/comments/14lz7j5/ntkaware_scaled_rope_allows_llama_models_to_have/}.
\newblock Accessed: 2023-08-25.

\bibitem[Rozière et~al.(2023)Rozière, Gehring, Gloeckle, Sootla, Gat, Tan,
  Adi, Liu, Remez, Rapin, Kozhevnikov, Evtimov, Bitton, Bhatt, Ferrer,
  Grattafiori, Xiong, Défossez, Copet, Azhar, Touvron, Martin, Usunier,
  Scialom, and Synnaeve]{codellama}
Baptiste Rozière, Jonas Gehring, Fabian Gloeckle, Sten Sootla, Itai Gat,
  Xiaoqing~Ellen Tan, Yossi Adi, Jingyu Liu, Tal Remez, Jérémy Rapin, Artyom
  Kozhevnikov, Ivan Evtimov, Joanna Bitton, Manish Bhatt, Cristian~Canton
  Ferrer, Aaron Grattafiori, Wenhan Xiong, Alexandre Défossez, Jade Copet,
  Faisal Azhar, Hugo Touvron, Louis Martin, Nicolas Usunier, Thomas Scialom,
  and Gabriel Synnaeve.
\newblock Code llama: Open foundation models for code, 2023.

\bibitem[Sakaguchi et~al.(2021)Sakaguchi, Bras, Bhagavatula, and
  Choi]{winogrande}
Keisuke Sakaguchi, Ronan~Le Bras, Chandra Bhagavatula, and Yejin Choi.
\newblock Winogrande: An adversarial winograd schema challenge at scale.
\newblock \emph{Communications of the ACM}, 64\penalty0 (9):\penalty0 99--106,
  2021.

\bibitem[Sap et~al.(2019)Sap, Rashkin, Chen, LeBras, and Choi]{siqa}
Maarten Sap, Hannah Rashkin, Derek Chen, Ronan LeBras, and Yejin Choi.
\newblock Socialiqa: Commonsense reasoning about social interactions.
\newblock \emph{arXiv preprint arXiv:1904.09728}, 2019.

\bibitem[Shaham et~al.(2023)Shaham, Ivgi, Efrat, Berant, and Levy]{zeroscrolls}
Uri Shaham, Maor Ivgi, Avia Efrat, Jonathan Berant, and Omer Levy.
\newblock Zeroscrolls: A zero-shot benchmark for long text understanding, 2023.

\bibitem[Su et~al.(2022)Su, Lu, Pan, Murtadha, Wen, and Liu]{rope}
Jianlin Su, Yu~Lu, Shengfeng Pan, Ahmed Murtadha, Bo~Wen, and Yunfeng Liu.
\newblock Roformer: Enhanced transformer with rotary position embedding, 2022.

\bibitem[Sun et~al.(2022)Sun, Dong, Patra, Ma, Huang, Benhaim, Chaudhary, Song,
  and Wei]{xpos}
Yutao Sun, Li~Dong, Barun Patra, Shuming Ma, Shaohan Huang, Alon Benhaim,
  Vishrav Chaudhary, Xia Song, and Furu Wei.
\newblock A length-extrapolatable transformer, 2022.

\bibitem[Talmor et~al.(2018)Talmor, Herzig, Lourie, and Berant]{commonsenseqa}
Alon Talmor, Jonathan Herzig, Nicholas Lourie, and Jonathan Berant.
\newblock Commonsenseqa: A question answering challenge targeting commonsense
  knowledge.
\newblock \emph{arXiv preprint arXiv:1811.00937}, 2018.

\bibitem[Together(2023)]{together}
Together.
\newblock Llama-2-7b-32k-instruct — and fine-tuning for llama-2 models with
  together api, 2023.
\newblock URL \url{https://together.ai/blog/llama-2-7b-32k-instruct}.

\bibitem[Touvron et~al.(2023)Touvron, Martin, Stone, Albert, Almahairi, Babaei,
  Bashlykov, Batra, Bhargava, Bhosale, et~al.]{llama2}
Hugo Touvron, Louis Martin, Kevin Stone, Peter Albert, Amjad Almahairi, Yasmine
  Babaei, Nikolay Bashlykov, Soumya Batra, Prajjwal Bhargava, Shruti Bhosale,
  et~al.
\newblock Llama 2: Open foundation and fine-tuned chat models.
\newblock \emph{arXiv preprint arXiv:2307.09288}, 2023.

\bibitem[Tworkowski et~al.(2023{\natexlab{a}})Tworkowski, Staniszewski, Pacek,
  Wu, Michalewski, and Miłoś]{focused}
Szymon Tworkowski, Konrad Staniszewski, Mikołaj Pacek, Yuhuai Wu, Henryk
  Michalewski, and Piotr Miłoś.
\newblock Focused transformer: Contrastive training for context scaling,
  2023{\natexlab{a}}.

\bibitem[Tworkowski et~al.(2023{\natexlab{b}})Tworkowski, Staniszewski, Pacek,
  Wu, Michalewski, and Miłoś]{long_llama}
Szymon Tworkowski, Konrad Staniszewski, Mikołaj Pacek, Yuhuai Wu, Henryk
  Michalewski, and Piotr Miłoś.
\newblock Focused transformer: Contrastive training for context scaling,
  2023{\natexlab{b}}.

\bibitem[Wang et~al.(2022)Wang, Kordi, Mishra, Liu, Smith, Khashabi, and
  Hajishirzi]{selfinst}
Yizhong Wang, Yeganeh Kordi, Swaroop Mishra, Alisa Liu, Noah~A Smith, Daniel
  Khashabi, and Hannaneh Hajishirzi.
\newblock Self-instruct: Aligning language model with self generated
  instructions.
\newblock \emph{arXiv preprint arXiv:2212.10560}, 2022.

\bibitem[Zellers et~al.(2019)Zellers, Holtzman, Bisk, Farhadi, and
  Choi]{hellaswag}
Rowan Zellers, Ari Holtzman, Yonatan Bisk, Ali Farhadi, and Yejin Choi.
\newblock Hellaswag: Can a machine really finish your sentence?, 2019.

\bibitem[Zhong et~al.(2021)Zhong, Yin, Yu, Zaidi, Mutuma, Jha, Awadallah,
  Celikyilmaz, Liu, Qiu, and Radev]{qmsum}
Ming Zhong, Da~Yin, Tao Yu, Ahmad Zaidi, Mutethia Mutuma, Rahul Jha,
  Ahmed~Hassan Awadallah, Asli Celikyilmaz, Yang Liu, Xipeng Qiu, and Dragomir
  Radev.
\newblock {QMS}um: A new benchmark for query-based multi-domain meeting
  summarization.
\newblock In \emph{Proceedings of the 2021 Conference of the North American
  Chapter of the Association for Computational Linguistics: Human Language
  Technologies}, pages 5905--5921, Online, June 2021. Association for
  Computational Linguistics.
\newblock \doi{10.18653/v1/2021.naacl-main.472}.
\newblock URL \url{https://aclanthology.org/2021.naacl-main.472}.

\end{thebibliography}
